\documentclass[10pt]{njulink}

\njulinksetup{
  type={Preprint},
  short-title={When to Use Subagents},
  subtitle={},
  keywords={}
}

\newcommand{\deepseekicon}{\raisebox{-0.3em}{\includegraphics[height=1.25em]{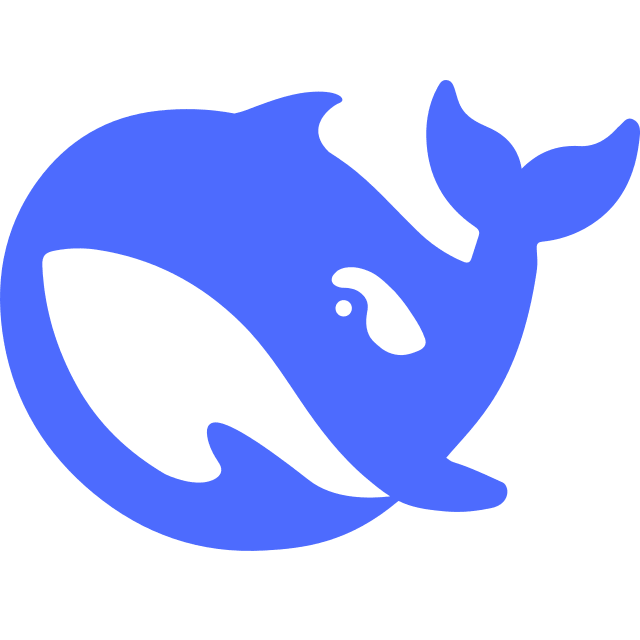}}}
\newcommand{\openaiicon}{\raisebox{-0.3em}{\includegraphics[height=1.25em]{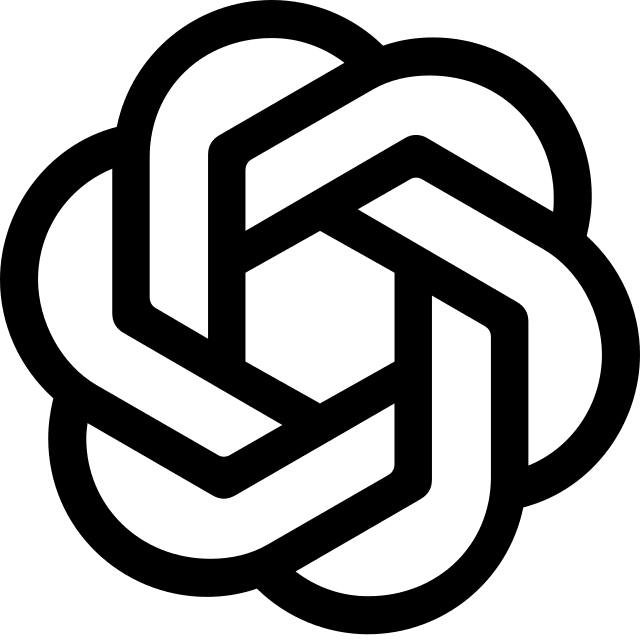}}}
\let\cite\citep

\title{When to Use Subagents}
\author[1]{Yishuo Yuan}
\author[1]{Yibo Wu}
\author[2]{Yihan Zhang}
\author[1]{Minyuan Sun}
\author[1]{Shenliang Li}
\author[1]{Xinkai Ma}
\author[1]{Yifan Li}
\author[1, \dag]{Jiaheng Liu}

\affiliation[1]{Nanjing University}
\affiliation[2]{Shanghai Jiao Tong University}
\contribution[\dag]{Corresponding author.}
\correspondence{Jiaheng Liu at \email{liujiaheng@nju.edu.cn}}
\date{October 2026}

\abstract{%
When should a task be delegated to subagents? We answer this at the granularity of individual trajectories. Modeling a trajectory as a dependency graph, we identify bridge edges, whose removal separates the trajectory into independent components, and show that the agent orchestration pattern should match this bridge-edge structure. Our experiments on long-horizon benchmarks indicate that subagents tend to benefit sparse bridge-edge structures, whereas tightly coupled trajectories are often better handled by a single agent. We instantiate this principle as SAIGE, a lightweight mechanism that spawns subagents only at predicted bridge edges and resolves their connections lazily, achieving a favorable trade-off between context efficiency and task performance where the structure admits decomposition. We further observe that scaling the agent pool or recursion depth does not reliably help, in line with the view that orchestration should follow the bridge-edge pattern.
}

\begin{document}
\maketitle

\section{Introduction}
\label{sec:intro}

The rapid advancement of large language models (LLMs) has reshaped the landscape of autonomous systems. As foundation models grow increasingly capable of handling long-horizon, complex tasks, the accompanying agent harnesses, such as Claude Code \citep{claudecode}, Codex \citep{codex}, and DeepSeek Harness \citep{deepseekharness}, have matured into sophisticated infrastructures that deliver robust, end-to-end execution in real-world scenarios. This progress has an important corollary for multi-agent collaboration: as individual agents and their harnesses become stronger and more long-horizon, the decomposition strategies that were once introduced to compensate for weaker agents deserve reexamination. The central question is no longer whether multiple agents can be assembled, but when delegating work to subagents is actually worthwhile.

Existing work often attributes the primary advantage of multi-agent systems to context isolation, which partitions a task to prevent context pollution and context rot \citep{contextfilter, contextrot}. While intuitively appealing, this explanation remains largely qualitative and lacks a structural account of when isolation pays off. In this work, we address this question at the granularity of individual trajectories. We formalize a task trajectory as a dependency graph over interaction steps and identify bridge edges, the sparse channels whose removal separates the trajectory into independent components. This representation yields a concrete design principle: the agent orchestration pattern should match the bridge-edge structure of the trajectory. We then examine empirically how this matching condition plays out across task structures, and find that subagents tend to benefit sparse bridge-edge structures, whereas tightly coupled trajectories are often better handled by a single agent.

Building on this principle, we propose SAIGE (Semantic-Aware Incremental Graph Evolution), a general multi-agent collaboration mechanism. SAIGE models collaboration as a dynamically evolving graph, where nodes are agent instances spawned on demand at predicted bridge edges, grounded in actual execution feedback, and edges encode semantic dependencies established lazily through content-based information retrieval. We validate our framework on long-horizon, complex task benchmarks, where SAIGE mitigates the decomposition penalty on tightly coupled tasks and improves both context efficiency and task performance on tasks whose dependency structure admits decomposition. Its advantage is therefore conditional on the bridge-edge structure of the task. Notably, our ablation study indicates that performance peaks at a modest agent budget consistent with the bridge-edge structure of the tasks, and that expanding the agent pool or deepening the recursion level does not consistently improve task performance.
\section{Related Work}
\subsection{Multi-Agent Collaboration}
\label{subsec:multi_agent_collaboration}

Multi-agent collaboration governs how agents coordinate, communicate, and align behaviors toward shared objectives, evolving from rigid predefined structures to adaptive, context-aware interactions. Role allocation ranges from static, pre-assigned responsibilities \citep{sheetagent, codedelegator} to dynamic selection or runtime instantiation based on task demands \citep{autogen, agentverse, roma}. Communication likewise spans explicit message passing via structured documents or natural language \citep{metagpt, talkhier, ecolang} and implicit inference from environmental cues \citep{generativeagents, latentmas}. Information flow further structures these interactions as sequential pipelines or parallel exploration \citep{hugginggpt, moa}, while interaction patterns ultimately manifest as cooperative or competitive dynamics \citep{proagent, gptbargaining}.

\subsection{Multi-Agent Topology}
\label{subsec:multi_agent_topology}

Multi-agent topology serves as a critical structural framework for governing agent execution, communication, and interactions. While early attempts often embedded structural organization implicitly \citep{chateval, autogen, dspy}, recent practices have explicitly represented multi-agent organizations as graphs \citep{dylan, gptswarm, macnet}. These topologies generally fall into three representative paradigms: centralized, distributed, and hybrid. Centralized topologies leverage a global coordinator for information aggregation and routing \citep{stackplanner, aop}; distributed topologies employ peer-to-peer communication and localized routing \citep{agentnet, symphony} to enhance autonomy and scalability; and hybrid topologies combine centralized strategic planning with decentralized execution \citep{chatdev, metagpt, agentorchestra}.
\section{Formalizing Multi-Agent Collaboration}
\label{sec:theory}

\subsection{Trajectory as a Dependency Directed Acyclic Graph}

Consider a task trajectory of length \(T\): \(\tau = \{x_1, x_2, \ldots, x_T\}\), where each \(x_t = (o_t, a_t)\) denotes the interaction tuple at step \(t\), consisting of the new observation \(o_t\) received and the action \(a_t\) taken by the agent. Let \(\mathcal{H}_t = \{x_1, \ldots, x_{t-1}\}\) denote the historical context available before generating \(x_t\). This history grows monotonically with \(t\), and the agent's decisions may depend on arbitrary subsets of this history.

To characterize these dependencies formally, we define a value function \(Q_t(a_t \mid o_t, \mathcal{H}_t)\) that represents the expected return of taking action \(a_t\) given historical context \(\mathcal{H}_t\). For any pair of nodes \(x_i, x_t \in \tau\) with \(i < t\), we say that \(x_t\) has a direct dependency on \(x_i\) if removing the interaction tuple \(x_i\) from the historical context causes a significant drop in the action value at step \(t\):

\begin{equation}
(x_i, x_t) \in \mathcal{E} 
\iff 
Q_t(a_t \mid o_t, \mathcal{H}_t \setminus \{x_i\}) < Q_t(a_t \mid o_t, \mathcal{H}_t) - \delta,
\end{equation}

where \(\mathcal{H}_t \setminus \{x_i\}\) denotes the historical context with the contribution of \(x_i\) removed, and \(\delta > 0\) is a predefined threshold. Since all edges point forward in time, the trajectory naturally induces a directed acyclic graph \(\mathcal{G} = (\mathcal{V}, \mathcal{E})\), with node set \(\mathcal{V} = \tau\). Each directed edge \((x_i, x_t) \in \mathcal{E} \subseteq \mathcal{V} \times \mathcal{V}\) indicates that node \(x_t\) directly depends on node \(x_i\).

We further assume that a task trajectory contains no redundant interactions. Specifically, every node in the trajectory contributes to at least one subsequent decision, implying that each node has out-degree at least one in the dependency graph, with the terminal node serving as the unique sink. Under this assumption, the induced DAG \(\mathcal{G}\) is weakly connected. This assumption is mild in practice, as interactions that do not influence any future step can be discarded from the trajectory without affecting the task outcome.

\begin{figure}[t]
\centering
\includegraphics[width=\linewidth]{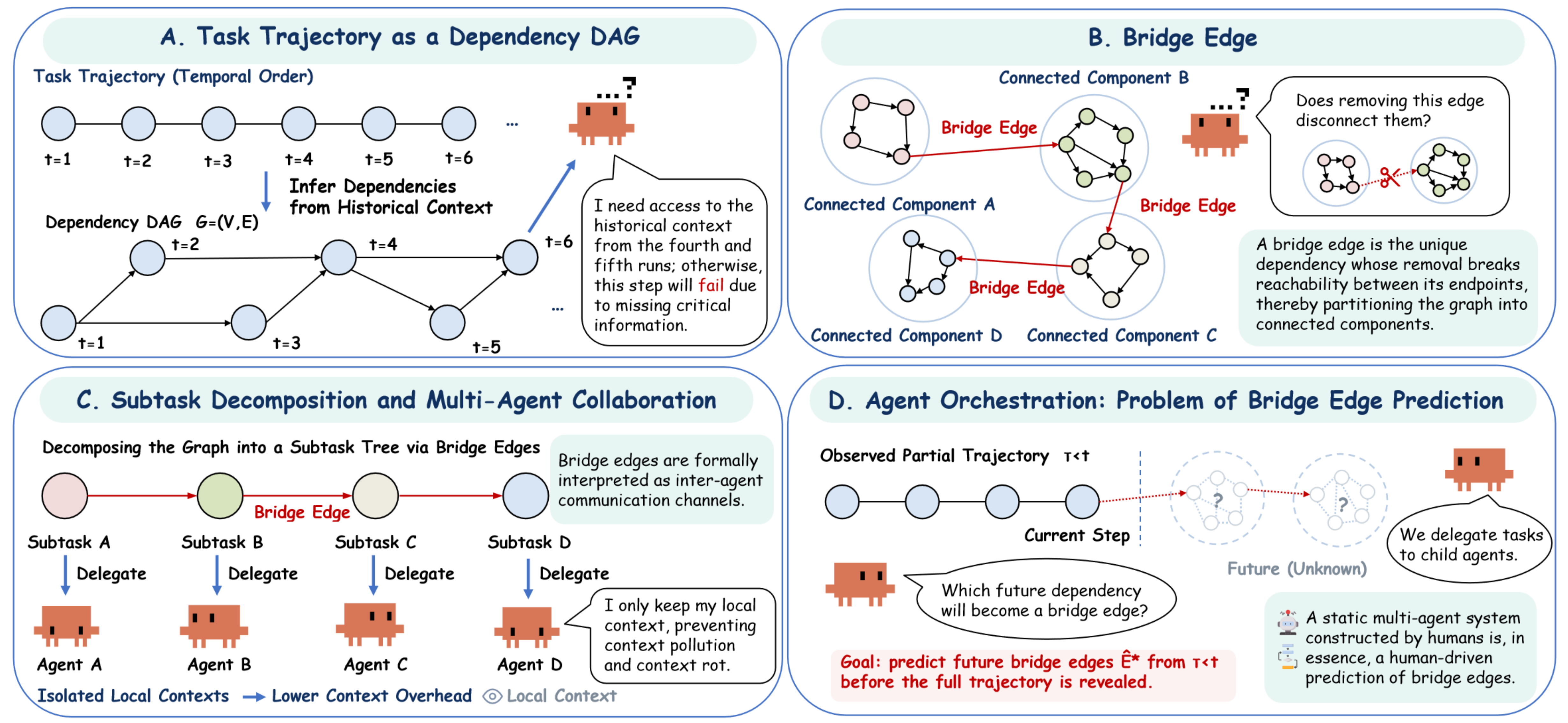}
\caption{Overview of the theoretical framework for multi-agent collaboration.}
\label{fig:theory}
\end{figure}

\subsection{Decomposing Trajectories into Subtasks via Bridge Edges}

Within the DAG \(\mathcal{G}\), certain edges serve as critical information channels connecting distinct regions of the graph. An edge \((x_i, x_t) \in \mathcal{E}\) is a \textbf{bridge edge} if its removal increases the number of weakly connected components of \(\mathcal{G}\).

Let \(\mathcal{E}^* \subseteq \mathcal{E}\) denote the set of all bridge edges. Removing all bridge edges from \(\mathcal{G}\) decomposes the graph into \(m\) weakly connected components \(\mathcal{G} \setminus \mathcal{E}^* = \mathcal{C}_1 \cup \cdots \cup \mathcal{C}_m\). The corresponding subtask trajectory is defined as the ordered sequence of nodes in \(\mathcal{C}_k\): \(\tau_{(k)} = \{x_{n_1}, x_{n_2}, \ldots, x_{n_{T_k}}\} \subseteq \tau\), where \(T_k = |\mathcal{C}_k|\), and the subtask trajectories collectively partition the original trajectory: \(\tau = \tau_{(1)} \cup \tau_{(2)} \cup \cdots \cup \tau_{(m)}\).

In multi-agent systems, this decomposition has a natural interpretation: each subtask \(\tau_{(k)}\) is assigned to an individual agent responsible for executing that segment of the trajectory. The bridge edges between components correspond to communication channels between the respective agents, enabling information flow across subtask boundaries when necessary.

Let \(\mathcal{V}_{\mathcal{S}} = \{\tau_{(1)}, \tau_{(2)}, \ldots, \tau_{(m)}\}\) denote the set of subtask trajectories. The directed edges between subtasks are induced by the bridge edges \(\mathcal{E}^*\). Since each bridge edge connects nodes from different components, this yields a well-defined relation:

\begin{equation}
(\tau_{(p)}, \tau_{(q)}) \in \mathcal{E}_{\mathcal{S}} 
\iff 
\exists (x_p, x_q) \in \mathcal{E}^* \text{ such that } x_p \in \tau_{(p)} \text{ and } x_q \in \tau_{(q)}.
\end{equation}

The pair \(\mathcal{T} = (\mathcal{V}_{\mathcal{S}}, \mathcal{E}_{\mathcal{S}})\) forms a subtask tree that captures information flow at the granularity of subtasks.

\subsection{Context Cost Reduction}

Let \( c(x) \) denote the non-negative computational overhead of including interaction tuple \( x \) in the historical context. For a subtask trajectory \( \tau_{(k)} = \{x_{n_1}, \ldots, x_{n_{T_k}}\} \), let \( \mathcal{B}_k \subseteq \mathcal{E}^* \) denote the set of bridge edges that enter \( \tau_{(k)} \), i.e.,
\(
\mathcal{B}_k = \{ (x_p, x_q) \in \mathcal{E}^* \mid x_q \in \tau_{(k)} \}.
\)
Let \( \operatorname{idx}(x, \tau) \) denote the index of node \( x \) within trajectory \( \tau \). The total context cost of executing subtask \( \tau_{(k)} \) is decomposed as
\begin{equation}
C(\tau_{(k)}) = C_I(\tau_{(k)}) + C_E(\tau_{(k)})
= \sum_{t=1}^{T_k} \sum_{i=1}^{t-1} c(x_{n_i})
+ \sum_{(x_p, x_q) \in \mathcal{B}_k} \sum_{i=\operatorname{idx}(x_q, \tau_{(k)})}^{T_k} c(x_p),
\label{eq:subtask_cost}
\end{equation}
where \( C_I(\tau_{(k)}) \) accounts for the cost of accumulating historical context from nodes within the same subtask, and \( C_E(\tau_{(k)}) \) accounts for the cost of pulling context \( x_p \) from predecessor subtasks via incoming bridge edges, starting from the position of the target node \( x_q \) through all subsequent steps.

For the original monolithic trajectory \( \tau \), there are no incoming bridge edges, so \( \mathcal{B} = \varnothing \), and the cost is simply:
\begin{equation}
C(\tau) = \sum_{t=1}^{T} \sum_{i=1}^{t-1} c(x_i).
\label{eq:total_cost}
\end{equation}

Our key claim is that, under idealized conditions where the complete task trajectory is known in advance, decomposing the trajectory into subtasks reduces the total context cost. The core of the proof is the following lemma, which shows that partitioning a trajectory along a single bridge edge does not increase the total context cost. The detailed proof of Lemma~1 is deferred to Appendix~\ref{app:proof}.

\begin{lemma}
\label{lem:context-cost}
For any trajectory \( \tau_{(S)} = \tau_{(L)} \cup \tau_{(R)} \) partitioned into two subtask trajectories by a single bridge edge \( (x_L, x_R) \) with \( x_L \in \tau_{(L)} \) and \( x_R \in \tau_{(R)} \), the sum of the context costs of the two subtask trajectories is no greater than the cost of the original trajectory:
\[
C(\tau_{(L)}) + C(\tau_{(R)}) \le C(\tau_{(S)}).
\]
\end{lemma}

Recursively applying Lemma~1 to \( \tau_{(L)} \) and \( \tau_{(R)} \) yields:
\[
C(\tau) \ge C(\tau_{(L)}) + C(\tau_{(R)}) \ge \cdots \ge \sum_{k=1}^{m} C(\tau_{(k)}).
\]
This establishes that, when the task trajectory is known and the decomposition is performed correctly, multi-agent decomposition reduces context cost. It should be noted that this analysis does not account for context compression or other mitigation strategies that may alter the cost structure.

\subsection{Orchestration as Bridge Edge Prediction}

The decomposition above is post-hoc: it presumes complete knowledge of the entire trajectory $\tau$ and its dependency structure. In practice, the trajectory unfolds sequentially, and the DAG $\mathcal{G}$ is not known in advance. The objective of multi-agent orchestration is therefore to predict the bridge edges $\mathcal{E}^*$ before or during execution, without access to the complete trajectory.

Let $\mathcal{E}^*$ denote the true set of bridge edges induced by the full trajectory, and let $\hat{\mathcal{E}}_t$ be the set of bridge edges predicted by the orchestration strategy at time $t$ based on the observed prefix $\tau_{<t}$. Since the true bridge edges $\mathcal{E}^*$ are only revealed after the entire trajectory is complete, the prediction $\hat{\mathcal{E}}_t$ concerns edges that lie in the future, i.e., edges $e_{i \to j}$ with $j > t$. The orchestration problem is to minimize the expected discrepancy between the true and predicted bridge edge sets:

\begin{equation}
\min_{\hat{\mathcal{E}}_t} \; \mathbb{E}\left[ \mathcal{L}(\mathcal{E}^*, \hat{\mathcal{E}}_t) \mid \tau_{<t} \right],
\end{equation}

where $\mathcal{L}$ is a loss function that measures the dissimilarity between two sets of directed edges over $\mathcal{V}$. A correct prediction partitions the trajectory into the true subtasks; an incorrect prediction results in suboptimal decomposition and diminished context cost reduction.


\section{Methodology}
\label{sec:method}

We propose Semantic-Aware Incremental Graph Evolution (SAIGE), a lightweight mechanism for decomposing tasks under trajectory uncertainty. SAIGE evolves the subtask tree incrementally during execution, growing the structure only when semantic dependencies are revealed by execution feedback.

\begin{figure}[t]
\centering
\includegraphics[width=\linewidth]{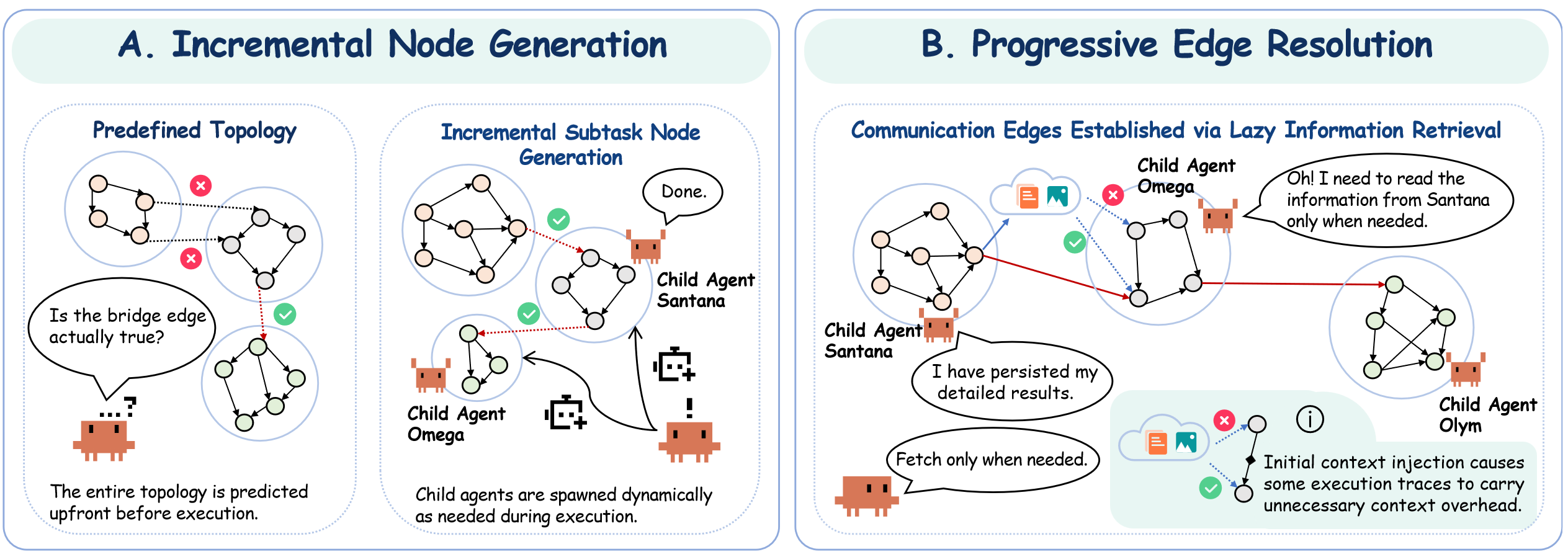}
\caption{SAIGE: incremental node generation via progressive subtask decomposition, and progressive edge resolution via on-demand information retrieval.}
\label{fig:method}
\end{figure}

\subsection{Incremental Node Generation}

The idealized decomposition in Section~\ref{sec:theory} assumes full trajectory knowledge, whereas real-world execution is sequential and uncertain. To bridge this gap, SAIGE formulates orchestration as a progressive bridge-edge decomposition problem. Instead of predicting a global topology upfront, it discovers dependency boundaries on-the-fly based on the execution history \(\mathcal{H}\).

At each step \(x_t\), SAIGE estimates the immediate downstream dependencies \(\mathcal{D}(x_t)\). A candidate bridge edge \((x_t, x_j)\) is identified when the out-degree of \(x_t\) is exactly 1, meaning it has only a single downstream consumer. To validate this edge, SAIGE checks if its removal increases the number of weakly connected components of the observed graph \(\mathcal{G}_t\). If confirmed as a true bridge, SAIGE extracts the downstream subtask \(\tau_{\text{sub}}\) and spawns a child agent via \textsc{SpawnAgent}. Otherwise, the parent continues monolithically.

\begin{algorithm}[h]
\caption{Progressive Bridge-Edge Decomposition}
\label{alg:progressive_decomp}
\begin{algorithmic}[1]
\State \(\text{attempt} \gets 0\)
\While{\(\text{attempt} \le \text{RetryLimit}\)}
    \State \(\mathcal{D}(x_t) \gets \textsc{EstimateDependencies}(x_t, \mathcal{H})\)
    \If{\textsc{OutDegree}\((x_t) == 1\)} \Comment{Check if out-degree of \(x_t\) is 1}
        \State \(x_j \gets \text{the sole node in } \mathcal{D}(x_t)\)
        \If{\(\textsc{IsBridge}((x_t, x_j), \mathcal{G}_t)\)}
            \State \(\tau_{\text{sub}} \gets \textsc{ExtractSubtask}(x_j, \mathcal{G}_t)\)
            \State \(\text{Thread} \gets \textsc{SpawnAgent}(\tau_{\text{sub}}, \text{Instruction}, \mathcal{H})\)
            \State \(\textsc{Wait}(\text{Thread})\)
            \State \(\text{status} \gets \textsc{GetExecutionResult}(\text{Thread})\)
            
            \If{\(\text{status} == \text{Success}\)}
                \State \textbf{break} \Comment{Sub-agent succeeded, exit retry loop}
            \EndIf
        \EndIf
    \EndIf
    \State \(\text{attempt} \gets \text{attempt} + 1\)
    \State \(\mathcal{H} \gets \mathcal{H} \cup \{ \textsc{BacktrackInfo}(x_t) \}\) \Comment{Record failure trace for retry}
\EndWhile
\end{algorithmic}
\end{algorithm}

To handle execution failures, SAIGE implements a retry mechanism, as shown in Algorithm~\ref{alg:progressive_decomp}. After spawning the child agent, the parent suspends execution via \textsc{Wait}. If the child thread returns a failure status, the parent records the failure trace into the history \(\mathcal{H}\) via \textsc{BacktrackInfo} and increments the attempt counter, triggering a retry. This process repeats until the child succeeds or the \text{RetryLimit} is reached.

By abstracting the structural checks into \textsc{EstimateDependencies} and \textsc{IsBridge}, SAIGE dynamically evolves the collaboration graph only when the task structure genuinely permits decomposition, avoiding the coordination penalties of pre-committed topologies.

\subsection{Progressive Edge Resolution}

Edges between subtask nodes are established through lazy information retrieval. Each agent writes its result upon completion and retrieves dependencies only when needed. Since dependencies become active only at the step where they are required rather than being injected into the context at the start of a subtask, sparse on-demand retrieval dominates push-based communication in context cost. More importantly, lazy information retrieval establishes communication channels dynamically instead of fixing them at decomposition time, allowing the collaboration graph to adapt to dependencies not fully determined in advance. This converts global topology prediction into localized retrieval decisions grounded in execution feedback.
\section{Experiments}
\label{sec:experiments}

\subsection{Experimental Setup}

\subsubsection{Benchmarks}

We evaluate on four benchmarks designed to assess agent capabilities on complex, long-horizon, realistic tasks. \textbf{Terminal Bench 2.1} \cite{terminalbench21} tests agents on hard, realistic command-line interface tasks, with resolution rate as the primary metric. \textbf{NL2Repo Bench} \cite{nl2repo} evaluates long-horizon repository generation from a single natural-language requirements document, using the average test pass rate across all tasks as the primary metric. \textbf{Deep Research Bench II} \cite{deepresearchbench} diagnoses deep research agents through rubrics derived from expert reports, scored as the fraction of rubrics passed. \textbf{AgentIF-OneDay} \cite{agentifoneday} is a task-level instruction-following benchmark for general AI agents in daily scenarios, employing instance-level, rubric-based scoring with binary evaluation, separated bonus/penalty items, and file-content alignment. Detailed descriptions of all benchmarks are provided in Appendix~\ref{app:benchmarks}.

\subsubsection{Baselines}

\textbf{Single-Agent} executes the full task within a single agent context and serves as the monolithic baseline. It runs on the native \textbf{Codex CLI} harness \cite{codex_cli}. \textbf{Task Decomposition and Agent Generation (TDAG)} \cite{tdag} is a multi-agent framework that decomposes complex tasks into smaller subtasks, assigns each subtask to a specifically generated subagent, and continuously summarizes successful execution patterns into reusable skills for future reference. \textbf{DynTaskMAS} \cite{taskdynmas} is a dynamic task graph-driven framework for asynchronous and parallel LLM-based multi-agent systems. \textbf{Graph-of-Agents (GoA)} \cite{goa} models multi-agent LLM communication through a graph-based framework, sampling relevant agents, constructing edges by evaluating response relevance, and aggregating responses via graph-based pooling. Further details on baseline implementations are provided in Appendix~\ref{app:baselines}.

\subsubsection{Implementation}

To ensure harness consistency, all methods are instantiated on the Codex CLI harness. For multi-agent methods, token usage is aggregated across all agents. To evaluate the generalizability of our approach, we instantiate all methods on two distinct foundation LLMs: \deepseekicon~\textbf{DeepSeek-V4-Pro} \cite{deepseekv4pro} and \openaiicon~\textbf{GPT-5.6 Luna} \cite{gptluna}. For all multi-agent methods evaluated in this work, we impose a uniform constraint on agent hierarchy: the \textbf{maximum recursion depth is set to 1}, meaning that spawned sub-agents are not permitted to recursively create further sub-agents. Additionally, the \textbf{maximum number of agents} is capped at 4. These constraints are imposed uniformly across all multi-agent methods to ensure system stability and prevent unbounded agent proliferation during execution. We examine the effect of relaxing these constraints in the ablation study (Section~\ref{sec:ablation} and Appendix~\ref{app:ablation_configs}).

\subsection{Main Results}

Tables~\ref{tab:main_deepseek} and \ref{tab:main_luna} present the main results across all four benchmarks using DeepSeek-V4-Pro and GPT-5.6-Luna, respectively. We report the primary metric along with aggregated input and output token counts. The primary metric is shown in normal font, and the token usage (Input / Output) is indicated in \textcolor{gray}{\textbf{gray}} at the bottom-right of each cell. The best primary metric for each benchmark is underlined.

\begin{table*}[h]
\small
\centering
\caption{Main results on \deepseekicon~DeepSeek-V4-Pro.}
\label{tab:main_deepseek}
\renewcommand{\arraystretch}{1.25} 
\begin{tabular}{lcccc}
\toprule
& \multicolumn{1}{c}{\textbf{Terminal Bench 2.1}} & \multicolumn{1}{c}{\textbf{NL2Repo}} & \multicolumn{1}{c}{\textbf{Deep Research Bench II}} & \multicolumn{1}{c}{\textbf{AgentIF-OneDay}} \\
\cmidrule(lr){2-2} \cmidrule(lr){3-3} \cmidrule(lr){4-4} \cmidrule(lr){5-5}
\textbf{Method} & \textbf{Resolution Rate \(\uparrow\)} & \textbf{Pass Rate \(\uparrow\)} & \textbf{Total Score \(\uparrow\)} & \textbf{Total Score \(\uparrow\)} \\
\midrule
Single-Agent & \underline{76.71} \textcolor{gray}{\scriptsize 475M/8M} & 48.74 \textcolor{gray}{\scriptsize 1,408M/10M} & 51.49 \textcolor{gray}{\scriptsize 2,805M/7M} & 53.57 \textcolor{gray}{\scriptsize 189M/2M} \\
TDAG         & 71.49 \textcolor{gray}{\scriptsize 542M/13M} & 47.37 \textcolor{gray}{\scriptsize 1,562M/15M} & 54.43 \textcolor{gray}{\scriptsize 3,895M/23M} & 59.94 \textcolor{gray}{\scriptsize 218M/3M} \\
DynTaskMAS   & 68.27 \textcolor{gray}{\scriptsize 533M/12M} & 44.15 \textcolor{gray}{\scriptsize 1,386M/15M} & 53.44 \textcolor{gray}{\scriptsize 3,240M/21M} & 65.85 \textcolor{gray}{\scriptsize 233M/4M} \\
GoA          & 67.47 \textcolor{gray}{\scriptsize 621M/16M} & 41.04 \textcolor{gray}{\scriptsize 1,380M/14M} & 48.40 \textcolor{gray}{\scriptsize 3,816M/23M} & 52.67 \textcolor{gray}{\scriptsize 222M/3M} \\
\rowcolor[HTML]{E8F6F3}
\textbf{SAIGE (Ours)} & 75.50 \textcolor{gray}{\scriptsize 501M/10M} & \underline{49.23} \textcolor{gray}{\scriptsize 1,282M/11M} & \underline{55.23} \textcolor{gray}{\scriptsize 2,745M/20M} & \underline{71.80} \textcolor{gray}{\scriptsize 185M/3M} \\
\bottomrule
\end{tabular}
\end{table*}

\begin{table*}[h]
\small
\centering
\caption{Main results on \openaiicon~GPT-5.6 Luna.}
\label{tab:main_luna}
\renewcommand{\arraystretch}{1.25} 
\begin{tabular}{lcccc}
\toprule
& \multicolumn{1}{c}{\textbf{Terminal Bench 2.1}} & \multicolumn{1}{c}{\textbf{NL2Repo}} & \multicolumn{1}{c}{\textbf{Deep Research Bench II}} & \multicolumn{1}{c}{\textbf{AgentIF-OneDay}} \\
\cmidrule(lr){2-2} \cmidrule(lr){3-3} \cmidrule(lr){4-4} \cmidrule(lr){5-5}
\textbf{Method} & \textbf{Resolution Rate \(\uparrow\)} & \textbf{Pass Rate \(\uparrow\)} & \textbf{Total Score \(\uparrow\)} & \textbf{Total Score \(\uparrow\)} \\
\midrule
Single-Agent & \underline{77.11} \textcolor{gray}{\scriptsize 813M/14M} & \underline{54.06} \textcolor{gray}{\scriptsize 1,953M/12M} & 52.75 \textcolor{gray}{\scriptsize 3.023M/9M} & 62.25 \textcolor{gray}{\scriptsize 236M/3M} \\
TDAG         & 73.09 \textcolor{gray}{\scriptsize 990M/17M} & 51.16 \textcolor{gray}{\scriptsize 2,210M/17M} & 54.18 \textcolor{gray}{\scriptsize 3,404M/26M} & 66.62 \textcolor{gray}{\scriptsize 242M/4M} \\
DynTaskMAS   & 71.89 \textcolor{gray}{\scriptsize 1021M/16M} & 52.27 \textcolor{gray}{\scriptsize 2,141M/16M} & 54.05 \textcolor{gray}{\scriptsize 3,250M/25M} & 67.18 \textcolor{gray}{\scriptsize 258M/3M} \\
GoA          & 69.07 \textcolor{gray}{\scriptsize 1165M/18M} & 45.38 \textcolor{gray}{\scriptsize 2,362M/18M} & 50.92 \textcolor{gray}{\scriptsize 3,672M/25M} & 62.83 \textcolor{gray}{\scriptsize 255M/3M} \\
\rowcolor[HTML]{FFF0D6}
\textbf{SAIGE (Ours)} & 75.10 \textcolor{gray}{\scriptsize 922M/16M} & 53.25 \textcolor{gray}{\scriptsize 2,088M/16M} & \underline{55.67} \textcolor{gray}{\scriptsize 2,891M/25M} & \underline{74.36} \textcolor{gray}{\scriptsize 224M/4M} \\
\bottomrule
\end{tabular}
\end{table*}

The results reveal two distinct regimes of multi-agent behavior. On Terminal Bench 2.1 and NL2Repo Bench, generic multi-agent methods consistently fall below the single-agent baseline, suggesting that decomposition fragments the dependency chain and that coordination overhead outweighs the benefit of context isolation. In contrast, on Deep Research Bench II and AgentIF-OneDay, decomposition appears more beneficial. SAIGE effectively avoids the decomposition penalty on the first two benchmarks, remaining comparable to the single-agent baseline while using fewer tokens than other multi-agent methods. Furthermore, it achieves the best scores on Deep Research Bench II and AgentIF-OneDay without the token inflation observed in other variants. GoA, by contrast, employs a more elaborate orchestration design yet yields weaker results, suggesting that added coordination complexity does not by itself translate into gains when unmatched to the task's dependency structure. These observations hold consistently across both foundation models. While GPT-5.6 Luna generally yields higher absolute scores, the relative performance trends remain identical across all four tasks, and SAIGE achieves a balanced trade-off between task performance and token overhead.

To understand the underlying mechanism, we conduct a post-hoc analysis of the execution trajectories, with step-level evidence provided in Appendix~\ref{app:case_study}. We find that code-centric tasks, such as Terminal Bench 2.1 and NL2Repo Bench, are inherently more coupled and exhibit fewer bridge edges, which explains why aggressive decomposition hurts performance. In contrast, search and daily tasks, such as Deep Research Bench II and AgentIF-OneDay, offer better separability and contain more bridge edges, making them naturally amenable to context isolation. Taken together, these results support our central claim that the benefit of multi-agent collaboration is bounded by task structure rather than universal.

\subsection{Analysis of Bridge-Edge Prediction Reliability}
\label{sec:bridge_edge_analysis}

To investigate whether SAIGE's performance gains stem from accurate dependency identification, we conduct a quantitative analysis of bridge-edge prediction reliability on the AgentIF-OneDay benchmark. We compare the multi-agent methods across two foundation models.

Figure~\ref{fig:bridge_edge_metrics} presents the Precision, Recall, and F1 scores for bridge-edge prediction across different methods. SAIGE generally outperforms the baselines on both DeepSeek-V4-Pro and GPT-5.6 Luna. For instance, on DeepSeek-V4-Pro, SAIGE achieves a competitive F1-score of 0.66, comparable to the best baseline TDAG, while demonstrating a distinct advantage in Precision. Similar trends are observed on GPT-5.6 Luna, indicating that our dynamic edge prediction mechanism remains robust across different foundation models. 

To provide a reference for these metrics, we perform a post-hoc analysis by establishing the Ground Truth for AgentIF-OneDay trajectories. We semantically annotate the true bridge edges based on the semantic dependency graph and compare them against the actual task partition points triggered by the multi-agent frameworks. The observed metrics in Figure~\ref{fig:bridge_edge_metrics} suggest that SAIGE tends to recover the true dependency boundaries more effectively than the baselines.

\begin{figure}[h]
    \centering
    \begin{subfigure}[b]{0.48\textwidth}
        \centering
        \includegraphics[width=\linewidth]{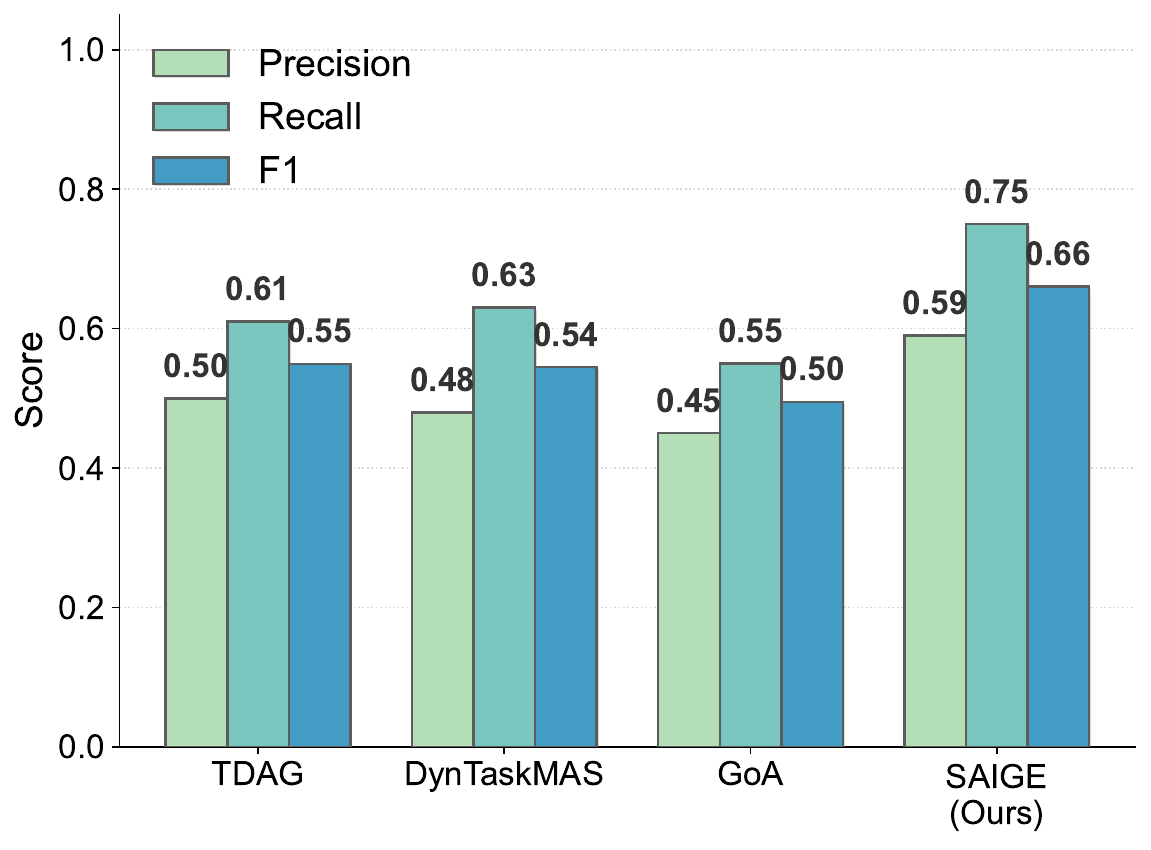}
        \caption*{(a) \deepseekicon~DeepSeek-V4-Pro}
        \label{fig:bridge_bar_ds}
    \end{subfigure}
    \hfill
    \begin{subfigure}[b]{0.48\textwidth}
        \centering
        \includegraphics[width=\linewidth]{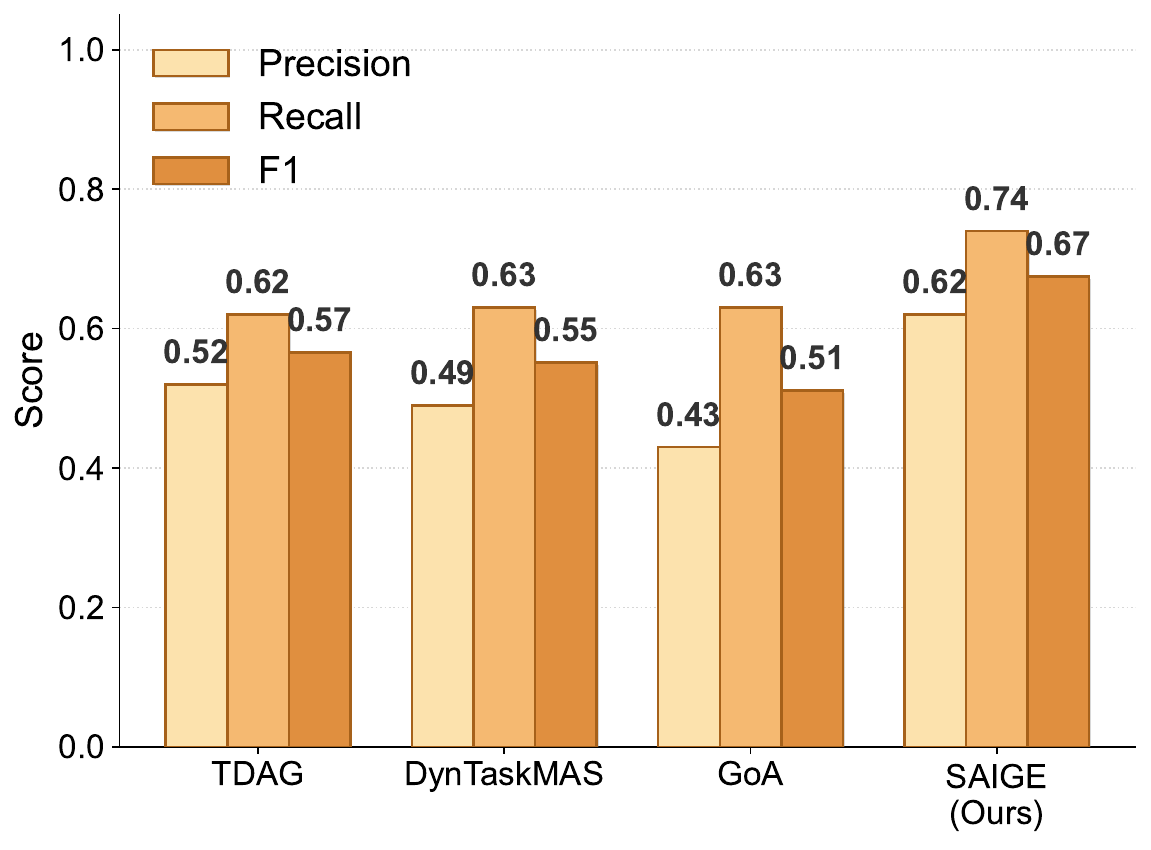}
        \caption*{(b) \openaiicon~GPT-5.6 Luna}
        \label{fig:bridge_bar_luna}
    \end{subfigure}
    \caption{Bridge-edge prediction metrics on AgentIF-OneDay. The average number of bridge edges per task is 2.82.}
    \label{fig:bridge_edge_metrics}
\end{figure}

\begin{figure*}[h]
    \centering
    \begin{subfigure}[b]{0.24\textwidth}
        \centering
        \includegraphics[width=\linewidth]{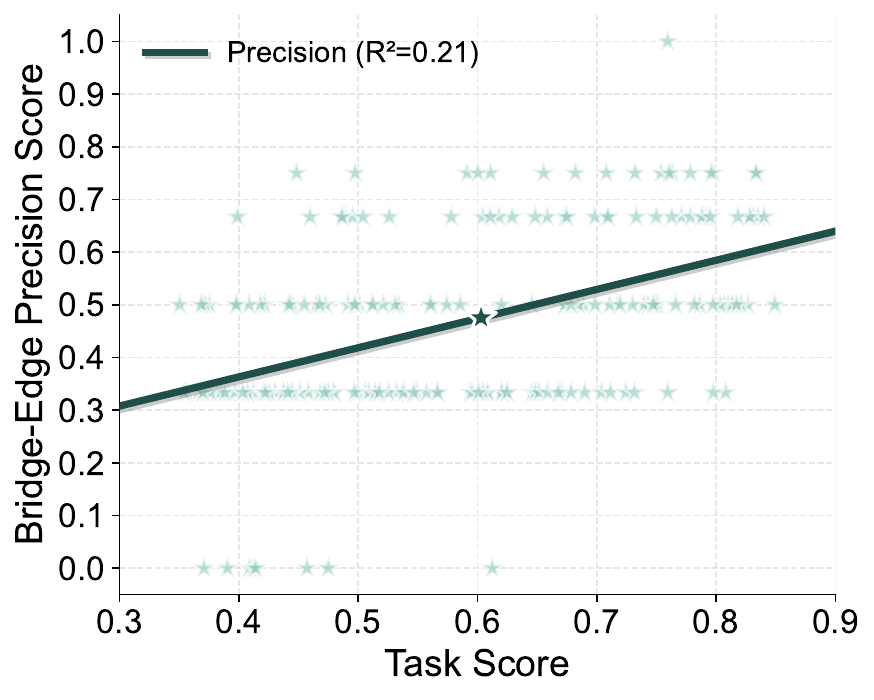}
        \caption*{(a) \deepseekicon~Regression: Precision}
    \end{subfigure}
    \hfill
    \begin{subfigure}[b]{0.24\textwidth}
        \centering
        \includegraphics[width=\linewidth]{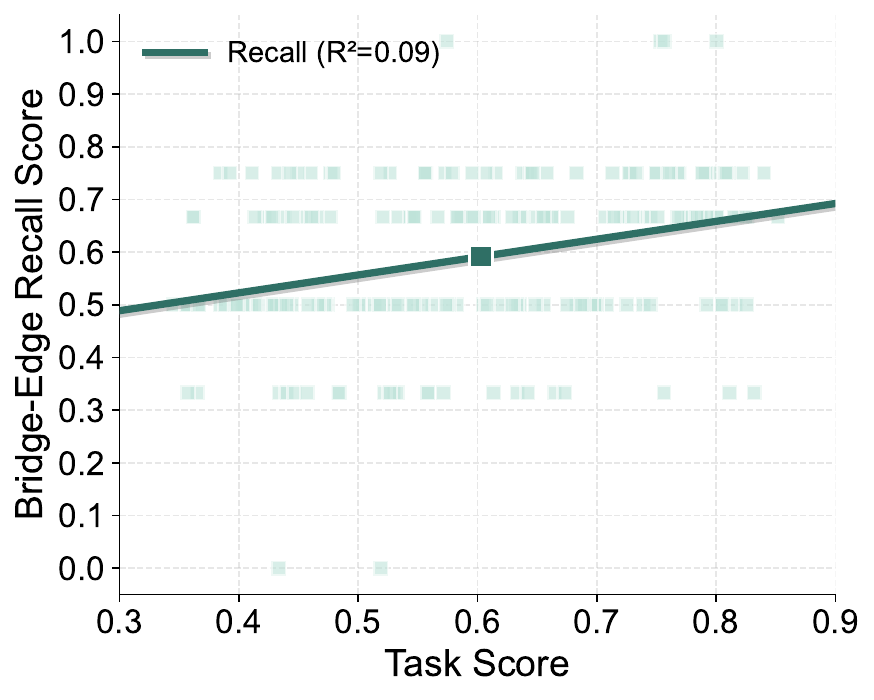}
        \caption*{(b) \deepseekicon~Regression: Recall}
    \end{subfigure}
    \hfill
    \begin{subfigure}[b]{0.24\textwidth}
        \centering
        \includegraphics[width=\linewidth]{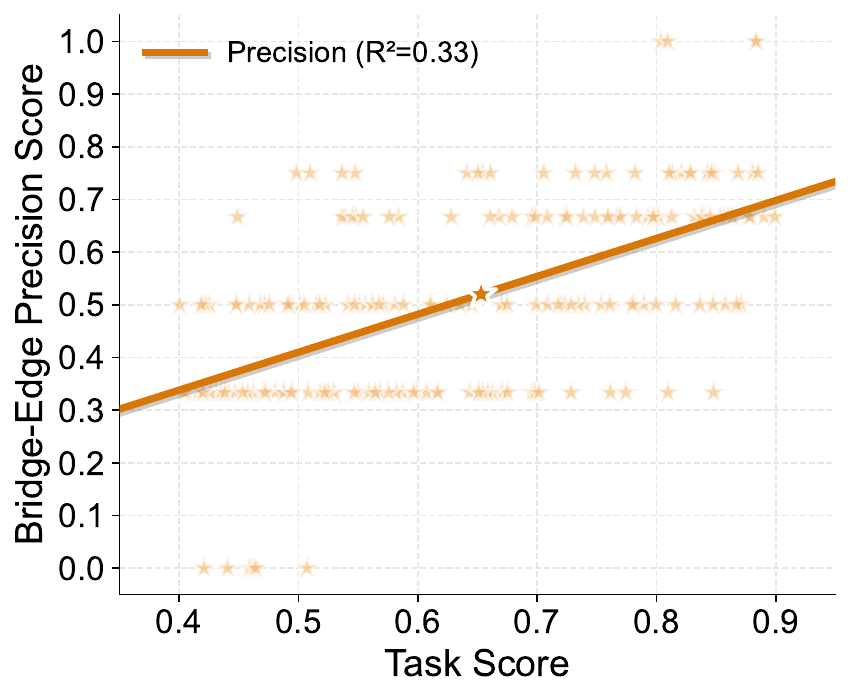}
        \caption*{(c) \openaiicon~Regression: Precision}
    \end{subfigure}
    \hfill
    \begin{subfigure}[b]{0.24\textwidth}
        \centering
        \includegraphics[width=\linewidth]{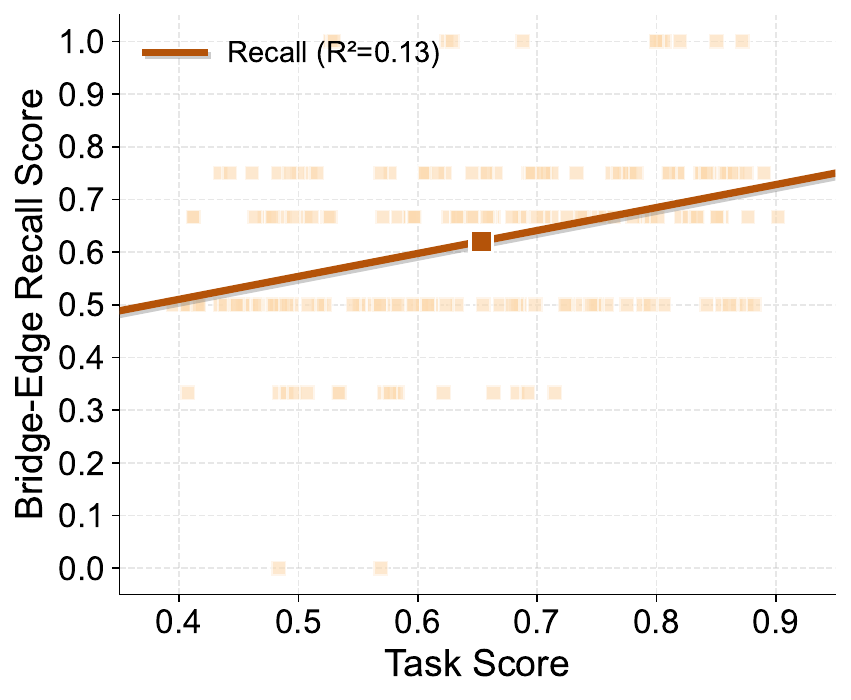}
        \caption*{(d) \openaiicon~Regression: Recall}
    \end{subfigure}
    
    \vspace{1em}
    \caption{Linear regression between the normalized overall task score and bridge-edge identification metrics on AgentIF-OneDay.}
    \label{fig:bridge_regression}
\end{figure*}

To understand how prediction reliability translates into actual task performance, we conduct a regression analysis. For each trajectory, we calculate a normalized overall task score, defined as the ratio of the primary task score to the token overhead. Figure~\ref{fig:bridge_regression} plots this normalized score against the bridge-edge Precision and Recall for both foundation models.

A notable observation is that Precision appears to exhibit a stronger positive correlation with the overall task score compared to Recall. While higher Recall helps ensure that fewer true dependencies are missed, lower Precision implies false-positive bridge edges. When a non-existent dependency boundary is erroneously predicted and decomposed, the resulting sub-agents may be forced to await information that never materializes. As discussed in Appendix~\ref{app:failure_taxonomy}, this can trigger local deadlock and negatively impact overall performance

\subsection{Ablation Study}
\label{sec:ablation}

We conduct ablation experiments on Deep Research Bench II and AgentIF-OneDay, using DeepSeek-V4-Pro as the foundation model, to examine the sensitivity of SAIGE to agent count, recursion depth, information retrieval, and dynamic graph construction. Tables~\ref{tab:ablation_dr} and~\ref{tab:ablation_agentif} report the primary benchmark score along with the aggregated token usage. The primary metric is shown in normal font, and the token usage (Input / Output) is indicated in \textcolor{gray}{\textbf{gray}} at the bottom-right. The default configuration is underlined.

We vary the maximum recursion depth \(D_{\max} \in \{1, 2, 3\}\) and the maximum number of agents \(N_{\max} \in \{2, 3, 4, 5, 6\}\). Note that \(N_{\max}=1\) corresponds to the single-agent baseline reported in the main results. Furthermore, ``w/o Retrieval'' denotes SAIGE with lazy information retrieval disabled under the default recursion depth \(D_{\max}=1\), while ``Static Topology'' represents a variant where the dependency graph is determined and fixed before execution, without deferred edge prediction. We choose \(N_{\max}=4\) and \(D_{\max}=1\) as our default configuration, as our pilot experiments indicated it represents the most favorable trade-off between performance and token efficiency.

\begin{table*}[h]
\small
\centering
\caption{Ablation results on Deep Research Bench II.}
\label{tab:ablation_dr}
\renewcommand{\arraystretch}{1.25}
\begin{tabular*}{\textwidth}{@{\extracolsep{\fill}}lccccc}
\toprule
& \multicolumn{3}{c}{\textbf{Max Recursion Depth}} & & \\
\cmidrule(lr){2-4}
\textbf{Config} & \(D_{\max}=1\) & \(D_{\max}=2\) & \(D_{\max}=3\) & \textbf{w/o Retrieval} & \textbf{Static Topology} \\
\midrule
\(N_{\max}=2\) & 51.18 \textcolor{gray}{\scriptsize 2,578M/15M} & --- & --- & 51.40 \textcolor{gray}{\scriptsize 2,801M/14M} & 51.89 \textcolor{gray}{\scriptsize 2,731M/15M} \\
\(N_{\max}=3\) & 54.95 \textcolor{gray}{\scriptsize 2,694M/18M} & 54.12 \textcolor{gray}{\scriptsize 2,744M/18M} & --- & 53.10 \textcolor{gray}{\scriptsize 3,063M/18M} & 54.67 \textcolor{gray}{\scriptsize 2,742M/20M} \\
\(N_{\max}=4\) & \underline{55.23} \textcolor{gray}{\scriptsize 2,745M/20M} & 54.88 \textcolor{gray}{\scriptsize 3,102M/19M} & 52.21 \textcolor{gray}{\scriptsize 2,993M/21M} & 53.54 \textcolor{gray}{\scriptsize 4,027M/24M} & 54.98 \textcolor{gray}{\scriptsize 2,919M/21M} \\
\(N_{\max}=5\) & 52.94 \textcolor{gray}{\scriptsize 3,913M/25M} & 52.07 \textcolor{gray}{\scriptsize 4,225M/24M} & 50.89 \textcolor{gray}{\scriptsize 4,502M/25M} & 53.85 \textcolor{gray}{\scriptsize 5,928M/26M} & 53.23 \textcolor{gray}{\scriptsize 4,212M/26M} \\
\(N_{\max}=6\) & 53.74 \textcolor{gray}{\scriptsize 6,312M/32M} & 48.94 \textcolor{gray}{\scriptsize 6,782M/34M} & 49.07 \textcolor{gray}{\scriptsize 6,590M/31M} & 51.22 \textcolor{gray}{\scriptsize 6,951M/34M} & 52.96 \textcolor{gray}{\scriptsize 6,820M/35M} \\
\bottomrule
\end{tabular*}
\end{table*}

\begin{table*}[h]
\small
\centering
\caption{Ablation results on AgentIF-OneDay.}
\label{tab:ablation_agentif}
\renewcommand{\arraystretch}{1.25}
\begin{tabular*}{\textwidth}{@{\extracolsep{\fill}}lccccc}
\toprule
& \multicolumn{3}{c}{\textbf{Max Recursion Depth}} & & \\
\cmidrule(lr){2-4}
\textbf{Config} & \(D_{\max}=1\) & \(D_{\max}=2\) & \(D_{\max}=3\) & \textbf{w/o Retrieval} & \textbf{Static Topology} \\
\midrule
\(N_{\max}=2\) & 62.46 \textcolor{gray}{\scriptsize 165M/2M} & --- & --- & 62.42 \textcolor{gray}{\scriptsize 172M/2M} & 64.23 \textcolor{gray}{\scriptsize 159M/2M} \\
\(N_{\max}=3\) & 69.27 \textcolor{gray}{\scriptsize 172M/2M} & 67.85 \textcolor{gray}{\scriptsize 205M/2M} & --- & 63.10 \textcolor{gray}{\scriptsize 183M/3M} & 67.51 \textcolor{gray}{\scriptsize 188M/2M} \\
\(N_{\max}=4\) & \underline{71.80} \textcolor{gray}{\scriptsize 185M/3M} & 66.39 \textcolor{gray}{\scriptsize 261M/4M} & 63.17 \textcolor{gray}{\scriptsize 272M/4M} & 66.22 \textcolor{gray}{\scriptsize 237M/3M} & 68.93 \textcolor{gray}{\scriptsize 195M/3M} \\
\(N_{\max}=5\) & 67.32 \textcolor{gray}{\scriptsize 239M/3M} & 58.33 \textcolor{gray}{\scriptsize 272M/4M} & 58.76 \textcolor{gray}{\scriptsize 270M/3M} & 59.15 \textcolor{gray}{\scriptsize 268M/3M} & 64.28 \textcolor{gray}{\scriptsize 250M/4M} \\
\(N_{\max}=6\) & 64.02 \textcolor{gray}{\scriptsize 352M/4M} & 60.16 \textcolor{gray}{\scriptsize 360M/4M} & 56.75 \textcolor{gray}{\scriptsize 371M/4M} & 53.50 \textcolor{gray}{\scriptsize 366M/3M} & 61.03 \textcolor{gray}{\scriptsize 381M/4M} \\
\bottomrule
\end{tabular*}
\end{table*}

Across the tested configurations, increasing the agent budget or deepening the recursion level does not consistently improve task performance. On both benchmarks, expanding the agent pool or allowing sub-agents to spawn further sub-agents tends to yield lower scores and higher token consumption. This suggests that beyond a certain point, the marginal agent adds coordination cost with limited additional problem-solving capacity. We also observe that performance generally peaks at \(N_{\max}=4\). We hypothesize that this can be plausibly attributed to the task structure itself, particularly on AgentIF-OneDay where the average number of bridge edges per task is approximately 2.82. This structural property naturally suggests an optimal decomposition granularity of around 3 to 4 subtasks, making \(N_{\max}=4\) a suitable default setting.

Disabling lazy information retrieval leads to a noticeable performance decline, particularly on AgentIF-OneDay, where the score drops from 71.80 to 66.22 at \(N_{\max}=4\). This validates the necessity of on-demand retrieval, as it allows the collaboration graph to adapt to dependencies that are not fully determined in advance.

The impact of dynamic graph construction varies significantly across benchmarks. On Deep Research Bench II, where task structures are relatively uniform, the performance of Static Topology remains close to that of our full SAIGE model across varying agent counts. In contrast, on AgentIF-OneDay, where task structures are highly diverse, the performance of Static Topology degrades rapidly as the number of agents increases, dropping from 68.93 at \(N_{\max}=4\) to 61.03 at \(N_{\max}=6\). This contrast suggests that when task trajectories are volatile, a pre-committed topology struggles to adapt, whereas execution-informed dynamic edge prediction effectively mitigates this issue.
\section{Conclusion}
\label{sec:conclusion}

In this work, we reexamine when a task should be delegated to subagents in the era of increasingly capable single-agent harnesses. By formalizing task trajectories as dependency graphs, we identify bridge edges and show that the agent orchestration pattern should match this bridge-edge structure. Our experiments on long-horizon benchmarks indicate that subagents tend to benefit sparse bridge-edge structures, whereas tightly coupled trajectories are often better handled by a single agent. Building on this principle, we propose SAIGE, a lightweight mechanism that spawns subagents only at predicted bridge edges and resolves their connections lazily, achieving a favorable trade-off between context efficiency and task performance where the structure admits decomposition. We further observe that scaling the agent pool or recursion depth does not reliably help, in line with the view that orchestration should follow the bridge-edge pattern rather than agent count. These findings suggest that the efficacy of subagents is closely tied to task structure.


\bibliographystyle{assets/plainnat}
\bibliography{references}

\clearpage
\beginappendix

\appendix
\section{Proofs}
\label{app:proof}

\textbf{Theorem~1.}
\textit{Let $\tau = \{x_1, x_2, \ldots, x_T\}$ be a task trajectory with no redundant interactions, and let $\mathcal{E}$ denote the set of directed edges defined by the dependency criterion. Then the graph $\mathcal{G} = (\mathcal{V}, \mathcal{E})$ with $\mathcal{V} = \tau$ is a weakly connected directed acyclic graph.}

\begin{proof}

For any edge $(x_i, x_t) \in \mathcal{E}$, the dependency criterion requires $i < t$, since the value function $Q_t$ is evaluated with respect to the historical context $\mathcal{H}_t = \{x_1, \ldots, x_{t-1}\}$. Thus every edge is oriented from a lower-index vertex to a higher-index vertex.

Define a mapping $\phi: \mathcal{V} \to \mathbb{N}$ by $\phi(x_i) = i$. For every edge $(x_i, x_t) \in \mathcal{E}$, we have
\[
\phi(x_i) = i < t = \phi(x_t).
\]
Therefore $\phi$ is a strict order-preserving labeling of the vertices, which is a well-known sufficient condition for acyclicity. Equivalently, any directed path
\[
x_{i_1} \to x_{i_2} \to \cdots \to x_{i_\ell}
\]
must satisfy
\[
i_1 < i_2 < \cdots < i_\ell,
\]
precluding the existence of a directed cycle. Hence $\mathcal{G}$ is a directed acyclic graph.

We next establish weak connectivity. By assumption, every node in $\tau$ contributes to at least one subsequent decision, implying that every vertex $x_i$ with $i < T$ has out-degree at least one, and the terminal vertex $x_T$ is the unique sink. Hence, for every vertex $x_i$ with $i < T$, there exists a directed path from $x_i$ to $x_T$: starting from $x_i$, each step follows an outgoing edge $(x_i, x_j)$ with $j > i$, and since indices strictly increase along edges and the trajectory is finite, the path must terminate at $x_T$.

Now consider any two vertices $x_i, x_j \in \mathcal{V}$. If $i = j$ the claim is trivial. Otherwise, there exists a directed path $P_i$ from $x_i$ to $x_T$ and a directed path $P_j$ from $x_j$ to $x_T$. Reversing the orientation of $P_j$ yields an undirected path from $x_T$ to $x_j$. Concatenating $P_i$ with the reversed $P_j$ gives an undirected walk from $x_i$ to $x_j$ in the underlying undirected graph of $\mathcal{G}$. Therefore any two vertices are connected in the underlying undirected graph, and $\mathcal{G}$ is weakly connected.
\end{proof}

\textbf{Theorem~2.}
\textit{Let \(\mathcal{G} = (\mathcal{V}, \mathcal{E})\) be a weakly connected directed acyclic graph, and let \(\mathcal{E}^* \subseteq \mathcal{E}\) be the set of bridge edges. Define \(\mathcal{V}_{\mathcal{S}}\) as the set of weakly connected components of \(\mathcal{G} \setminus \mathcal{E}^*\), and let
\[
\mathcal{E}_{\mathcal{S}} = \{ (\mathcal{C}_p, \mathcal{C}_q) \mid \exists (u, v) \in \mathcal{E}^* \text{ with } u \in \mathcal{C}_p, v \in \mathcal{C}_q \}.
\]
Then \(\mathcal{T} = (\mathcal{V}_{\mathcal{S}}, \mathcal{E}_{\mathcal{S}})\) is a tree.}

\begin{proof}

We first prove that \(\mathcal{T}\) is connected. Since \(\mathcal{G}\) is weakly connected, its underlying undirected graph is connected. Contracting each weakly connected component \(\mathcal{C}_i \in \mathcal{V}_{\mathcal{S}}\) to a single vertex preserves weak connectivity of the quotient graph. The edge set of this quotient graph is precisely \(\mathcal{E}_{\mathcal{S}}\), as all edges between distinct components are bridge edges. Hence \(\mathcal{T}\) is weakly connected.

We next prove that \(\mathcal{T}\) is acyclic. Suppose, for contradiction, that \(\mathcal{T}\) contains an undirected cycle \(C_{\mathcal{S}}\) of length \(\ell \ge 2\), with vertices
\[
\mathcal{C}_{i_1}, \mathcal{C}_{i_2}, \ldots, \mathcal{C}_{i_\ell}
\]
and edges
\[
e_1, e_2, \ldots, e_\ell \in \mathcal{E}_{\mathcal{S}},
\]
where \(e_j = (\mathcal{C}_{i_j}, \mathcal{C}_{i_{j+1}})\) with indices taken modulo \(\ell\).

By the definition of \(\mathcal{E}_{\mathcal{S}}\), each \(e_j\) corresponds to a bridge edge \((u_j, v_j) \in \mathcal{E}^*\) with
\[
u_j \in \mathcal{C}_{i_j} \quad \text{and} \quad v_j \in \mathcal{C}_{i_{j+1}}.
\]
Since each \(\mathcal{C}_{i_j}\) is weakly connected, there exists an undirected path \(P_j\) within \(\mathcal{C}_{i_{j+1}}\) connecting \(v_j\) to \(u_{j+1}\), with indices modulo \(\ell\).

Concatenating the bridge edge \((u_j, v_j)\) with the path \(P_j\) for each \(j\) yields an undirected cycle in \(\mathcal{G}\) that contains the bridge edge \((u_1, v_1)\) but does not traverse it twice. The removal of \((u_1, v_1)\) therefore does not disconnect \(u_1\) from \(v_1\), as the alternative route through
\[
P_1, e_2, P_2, \ldots, e_\ell
\]
remains. This contradicts the defining property of a bridge edge, namely that its removal destroys all paths between its endpoints. Therefore \(\mathcal{T}\) contains no undirected cycle.

A weakly connected graph with no undirected cycles is a tree, completing the proof.
\end{proof}

\textbf{Lemma~1.}
\textit{For any trajectory \( \tau_{(S)} = \tau_{(L)} \cup \tau_{(R)} \) partitioned into two subtask trajectories by a single bridge edge \( (x_L, x_R) \) with \( x_L \in \tau_{(L)} \) and \( x_R \in \tau_{(R)} \), the sum of the context costs of the two subtask trajectories is no greater than the cost of the original trajectory:
\[
C(\tau_{(L)}) + C(\tau_{(R)}) \le C(\tau_{(S)}).
\]}

\begin{proof}
We proceed by analyzing the internal and external context costs separately, then combine the two bounds to establish the result.

\paragraph{Setup and notation.}
Both subsequences preserve the original temporal order, so every node in $\tau_{(L)}$ precedes every node in $\tau_{(R)}$. For any trajectory $\tau$ and node $x_i \in \tau$, let $\operatorname{idx}(x_i, \tau)$ denote the position of $x_i$ in $\tau$.

For any trajectory $\tau$, the internal context cost can be rewritten as
\[
C_I(\tau) = \sum_{x_i \in \tau} c(x_i) \cdot \bigl|\{x_t \in \tau : \operatorname{idx}(x_i, \tau) < \operatorname{idx}(x_t, \tau)\}\bigr|,
\]
which charges each node $x_i$ for every subsequent node in the same trajectory that must retain $x_i$ in its historical context.

For any trajectory $\tau$ with incoming bridge edge set $\mathcal{B}(\tau) \subseteq \mathcal{E}^*$, the external context cost can be rewritten as
\[
C_E(\tau) = \sum_{(x_p, x_q) \in \mathcal{B}(\tau)} c(x_p) \cdot \bigl|\{x_t \in \tau : \operatorname{idx}(x_q, \tau) \le \operatorname{idx}(x_t, \tau)\}\bigr|,
\]
which charges each predecessor node $x_p$ for every node in $\tau$ that must retain $x_p$ after it is pulled across the bridge edge $(x_p, x_q)$.

The total context cost is $C(\tau) = C_I(\tau) + C_E(\tau)$.

\paragraph{Internal cost difference.}
Since $\tau_{(S)} = \tau_{(L)} \cup \tau_{(R)}$, the internal cost of $\tau_{(S)}$ decomposes as
\[
\begin{aligned}
C_I(\tau_{(S)}) &= C_I(\tau_{(L)}) + C_I(\tau_{(R)}) \\
&\quad + \sum_{x_i \in \tau_{(L)}} c(x_i) \cdot \bigl|\{x_t \in \tau_{(R)} : \operatorname{idx}(x_i, \tau_{(S)}) < \operatorname{idx}(x_t, \tau_{(S)})\}\bigr| \\
&\quad + \sum_{x_j \in \tau_{(R)}} c(x_j) \cdot \bigl|\{x_t \in \tau_{(L)} : \operatorname{idx}(x_j, \tau_{(S)}) < \operatorname{idx}(x_t, \tau_{(S)})\}\bigr|.
\end{aligned}
\]
The first extra term accounts for the fact that in the monolithic trajectory, nodes in $\tau_{(R)}$ must retain preceding nodes in $\tau_{(L)}$ in their context. The second extra term is symmetric: it accounts for nodes in $\tau_{(R)}$ whose context must be retained by subsequent nodes in $\tau_{(L)}$. Since both extra terms are non-negative, dropping the second term yields the lower bound
\begin{equation}
C_I(\tau_{(S)}) - C_I(\tau_{(L)}) - C_I(\tau_{(R)})
\ge \sum_{x_i \in \tau_{(L)}} c(x_i) \cdot \bigl|\{x_t \in \tau_{(R)} : \operatorname{idx}(x_i, \tau_{(S)}) < \operatorname{idx}(x_t, \tau_{(S)})\}\bigr|.
\label{eq:I_diff}
\end{equation}

\paragraph{External cost difference.}
Let $\mathcal{B}_S$ be the set of bridge edges entering $\tau_{(S)}$. After partitioning by the bridge edge $(x_L, x_R)$, the incoming bridge edge sets for the two subtasks are
\[
\mathcal{B}_L = \{ (x_p, x_q) \in \mathcal{B}_S \mid x_q \in \tau_{(L)} \},
\qquad
\mathcal{B}_R = \{ (x_p, x_q) \in \mathcal{B}_S \mid x_q \in \tau_{(R)} \} \cup \{(x_L, x_R)\}.
\]
We claim that $\mathcal{B}_S$ decomposes disjointly as
\[
\mathcal{B}_S = \mathcal{B}_L \cup \bigl(\mathcal{B}_R \setminus \{(x_L, x_R)\}\bigr),
\]
with $\mathcal{B}_L \cap \bigl(\mathcal{B}_R \setminus \{(x_L, x_R)\}\bigr) = \varnothing$. Indeed, suppose for contradiction that there exists a bridge edge $(x_p, x_q) \in \mathcal{B}_L \cap \bigl(\mathcal{B}_R \setminus \{(x_L, x_R)\}\bigr)$. Then $x_q$ would belong to both $\tau_{(L)}$ and $\tau_{(R)}$, contradicting the fact that $\tau_{(L)}$ and $\tau_{(R)}$ are disjoint subtask trajectories obtained by removing the bridge edge $(x_L, x_R)$.

The external cost of the original trajectory is
\[
\begin{aligned}
C_E(\tau_{(S)})
&= \sum_{(x_p, x_q) \in \mathcal{B}_L} c(x_p) \cdot \bigl|\{x_t \in \tau_{(S)} : \operatorname{idx}(x_q, \tau_{(S)}) \le \operatorname{idx}(x_t, \tau_{(S)})\}\bigr| \\
&\quad + \sum_{(x_p, x_q) \in \mathcal{B}_R \setminus \{(x_L, x_R)\}} c(x_p) \cdot \bigl|\{x_t \in \tau_{(S)} : \operatorname{idx}(x_q, \tau_{(S)}) \le \operatorname{idx}(x_t, \tau_{(S)})\}\bigr|.
\end{aligned}
\]

The external costs of the two subtasks are
\[
\begin{aligned}
C_E(\tau_{(L)}) + C_E(\tau_{(R)})
&= \sum_{(x_p, x_q) \in \mathcal{B}_L} c(x_p) \cdot \bigl|\{x_t \in \tau_{(L)} : \operatorname{idx}(x_q, \tau_{(L)}) \le \operatorname{idx}(x_t, \tau_{(L)})\}\bigr| \\
&\quad + \sum_{(x_p, x_q) \in \mathcal{B}_R \setminus \{(x_L, x_R)\}} c(x_p) \cdot \bigl|\{x_t \in \tau_{(R)} : \operatorname{idx}(x_q, \tau_{(R)}) \le \operatorname{idx}(x_t, \tau_{(R)})\}\bigr| \\
&\quad + c(x_L) \cdot \bigl|\{x_t \in \tau_{(R)} : \operatorname{idx}(x_R, \tau_{(R)}) \le \operatorname{idx}(x_t, \tau_{(R)})\}\bigr|.
\end{aligned}
\]

For any $(x_p, x_q) \in \mathcal{B}_L$, since $\tau_{(L)} \subseteq \tau_{(S)}$ and the relative positions of nodes within $\tau_{(L)}$ are preserved in $\tau_{(S)}$, we have $\operatorname{idx}(x_q, \tau_{(L)}) = \operatorname{idx}(x_q, \tau_{(S)})$. Consequently,
\[
\{x_t \in \tau_{(L)} : \operatorname{idx}(x_q, \tau_{(L)}) \le \operatorname{idx}(x_t, \tau_{(L)})\}
\subseteq
\{x_t \in \tau_{(S)} : \operatorname{idx}(x_q, \tau_{(S)}) \le \operatorname{idx}(x_t, \tau_{(S)})\},
\]
and therefore
\[
\begin{aligned}
&\bigl|\{x_t \in \tau_{(L)} : \operatorname{idx}(x_q, \tau_{(L)}) \le \operatorname{idx}(x_t, \tau_{(L)})\}\bigr| \\
&\qquad \le
\bigl|\{x_t \in \tau_{(S)} : \operatorname{idx}(x_q, \tau_{(S)}) \le \operatorname{idx}(x_t, \tau_{(S)})\}\bigr|.
\end{aligned}
\]
The same argument applies to any $(x_p, x_q) \in \mathcal{B}_R \setminus \{(x_L, x_R)\}$, yielding
\[
\begin{aligned}
&\bigl|\{x_t \in \tau_{(R)} : \operatorname{idx}(x_q, \tau_{(R)}) \le \operatorname{idx}(x_t, \tau_{(R)})\}\bigr| \\
&\qquad \le
\bigl|\{x_t \in \tau_{(S)} : \operatorname{idx}(x_q, \tau_{(S)}) \le \operatorname{idx}(x_t, \tau_{(S)})\}\bigr|.
\end{aligned}
\]

Thus, comparing the expressions for $C_E(\tau_{(S)})$ and $C_E(\tau_{(L)}) + C_E(\tau_{(R)})$ term by term, the only term present in the subtask costs but absent from the original cost is the newly introduced bridge edge cost
\[
c(x_L) \cdot \bigl|\{x_t \in \tau_{(R)} : \operatorname{idx}(x_R, \tau_{(R)}) \le \operatorname{idx}(x_t, \tau_{(R)})\}\bigr|.
\]
Therefore,
\begin{equation}
C_E(\tau_{(S)}) - C_E(\tau_{(L)}) - C_E(\tau_{(R)})
\ge -c(x_L) \cdot \bigl|\{x_t \in \tau_{(R)} : \operatorname{idx}(x_R, \tau_{(R)}) \le \operatorname{idx}(x_t, \tau_{(R)})\}\bigr|.
\label{eq:E_diff}
\end{equation}

\paragraph{Combining the bounds.}
Combining (\ref{eq:I_diff}) and (\ref{eq:E_diff}), we have
\[
\begin{aligned}
& C(\tau_{(S)}) - C(\tau_{(L)}) - C(\tau_{(R)}) \\
&= \bigl[C_I(\tau_{(S)}) - C_I(\tau_{(L)}) - C_I(\tau_{(R)})\bigr] \\
&\quad + \bigl[C_E(\tau_{(S)}) - C_E(\tau_{(L)}) - C_E(\tau_{(R)})\bigr] \\
&\ge \sum_{x_i \in \tau_{(L)}} c(x_i) \cdot \bigl|\{x_t \in \tau_{(R)} : \operatorname{idx}(x_i, \tau_{(S)}) < \operatorname{idx}(x_t, \tau_{(S)})\}\bigr| \\
&\quad - c(x_L) \cdot \bigl|\{x_t \in \tau_{(R)} : \operatorname{idx}(x_R, \tau_{(R)}) \le \operatorname{idx}(x_t, \tau_{(R)})\}\bigr|.
\end{aligned}
\]

Since $x_L \in \tau_{(L)}$, the first sum contains the term corresponding to $x_i = x_L$:
\[
\begin{aligned}
&\sum_{x_i \in \tau_{(L)}} c(x_i) \cdot \bigl|\{x_t \in \tau_{(R)} : \operatorname{idx}(x_i, \tau_{(S)}) < \operatorname{idx}(x_t, \tau_{(S)})\}\bigr| \\
&\qquad \ge c(x_L) \cdot \bigl|\{x_t \in \tau_{(R)} : \operatorname{idx}(x_L, \tau_{(S)}) < \operatorname{idx}(x_t, \tau_{(S)})\}\bigr|.
\end{aligned}
\]

Because $(x_L, x_R)$ is a bridge edge, we have $\operatorname{idx}(x_L, \tau_{(S)}) < \operatorname{idx}(x_R, \tau_{(S)})$, and hence
\[
\{x_t \in \tau_{(R)} : \operatorname{idx}(x_R, \tau_{(S)}) \le \operatorname{idx}(x_t, \tau_{(S)})\}
\subseteq
\{x_t \in \tau_{(R)} : \operatorname{idx}(x_L, \tau_{(S)}) < \operatorname{idx}(x_t, \tau_{(S)})\}.
\]
Since $\operatorname{idx}(x_R, \tau_{(R)}) = \operatorname{idx}(x_R, \tau_{(S)})$, this implies
\[
\begin{aligned}
&\bigl|\{x_t \in \tau_{(R)} : \operatorname{idx}(x_R, \tau_{(R)}) \le \operatorname{idx}(x_t, \tau_{(R)})\}\bigr| \\
&\qquad \le
\bigl|\{x_t \in \tau_{(R)} : \operatorname{idx}(x_L, \tau_{(S)}) < \operatorname{idx}(x_t, \tau_{(S)})\}\bigr|.
\end{aligned}
\]

Therefore, the first term in the lower bound dominates the second term, and we obtain
\[
\begin{aligned}
& C(\tau_{(S)}) - C(\tau_{(L)}) - C(\tau_{(R)}) \\
&\ge c(x_L) \cdot \bigl|\{x_t \in \tau_{(R)} : \operatorname{idx}(x_L, \tau_{(S)}) < \operatorname{idx}(x_t, \tau_{(S)})\}\bigr| \\
&\quad - c(x_L) \cdot \bigl|\{x_t \in \tau_{(R)} : \operatorname{idx}(x_R, \tau_{(R)}) \le \operatorname{idx}(x_t, \tau_{(R)})\}\bigr| \\
&\ge 0.
\end{aligned}
\]

Thus,
\[
C(\tau_{(L)}) + C(\tau_{(R)}) \le C(\tau_{(S)}),
\]
which completes the proof.
\end{proof}

\section{Experimental Details}

\subsection{Benchmark Details}
\label{app:benchmarks}

\textbf{Terminal Bench 2.1} \cite{terminalbench21} is a curated hard benchmark composed of tasks in computer terminal environments inspired by real-world workflows. Each task features a unique environment, a human-written solution, and comprehensive tests for verification. Because the framework is interactive, agent and model performance are difficult to decouple; many agent scaffolds are engineered to accommodate the tendencies of specific models, and agents are not explicitly required to use a terminal as their sole tool. To remain faithful to the benchmark's premise, a neutral testbed with a single headless terminal is adopted, and tasks are completed using only Bash commands. The primary metric is the resolution rate. In our evaluation, each task is attempted three times, and the resolution rate is computed as the number of tasks completed in a run divided by the total number of tasks, averaged over the three runs.

\textbf{NL2Repo Bench} \cite{nl2repo} evaluates repository-level coding abilities with verifiable ground truth. Tasks are derived from real-world Python libraries characterized by modular architectures and authoritative \texttt{pytest} suites. Agents receive only a single natural-language specification and must reconstruct the complete repository from scratch, including file structures and functional logic. Correctness is strictly measured by executing the generated code against the original upstream tests. Each task is scored by the fraction of unit tests it passes, and the final benchmark score is the average of these per-task pass rates across all tasks.

\textbf{Deep Research Bench II} \cite{deepresearchbench} diagnoses deep research agents through rubrics derived from expert reports. The benchmark evaluates agents on information recall, analysis, and presentation, with rubrics constructed through a four-stage pipeline of LLM extraction, self-evaluation iteration, manual revision, and expert review. For each task, the score is the fraction of rubrics it passes, and the final benchmark score is computed as the total number of rubrics passed divided by the total number of rubrics across all tasks.

\textbf{AgentIF-OneDay} \cite{agentifoneday} is a task-level instruction-following benchmark for general AI agents in daily scenarios. It employs instance-level, rubric-based scoring with three key properties: binary scoring, where each rubric item is strictly evaluated as satisfied or not to ensure objectivity; distinction between bonus and penalty items, separating capability assessment from error rates; and file-content alignment, enabling direct evaluation of agent-generated files alongside textual output. Let $N$ be the total number of problems. For the $i$-th problem, let $S_i^+$ and $S_i^-$ denote the sum of satisfied bonus points and triggered penalty points, and $S_i^{\max}$ the maximum achievable score. The normalized score is
\begin{equation}
s_i = \frac{\max(0, S_i^+ - S_i^-)}{S_i^{\max}},
\end{equation}
and the final benchmark score is the mean of these normalized scores:
\begin{equation}
\mathrm{Score}_{\mathrm{final}} = \frac{1}{N} \sum_{i=1}^{N} s_i.
\end{equation}
Because AgentIF-OneDay involves multimodal requirements such as image editing and in-depth research, while DeepSeek-V4-Pro \cite{deepseekv4pro} does not natively support multimodal input, we provide a multimodal skill that allows the agent to invoke DeepSeek-V4-Flash-Vision-Exp \cite{deepseekv4flashvisionexp} for image understanding. This skill is implemented as a script that the agent can invoke directly, so the agent can decide when to delegate visual perception to the vision model without altering the benchmark's evaluation protocol.

\subsection{Baseline Details}
\label{app:baselines}

\textbf{Single-Agent} executes the full task within a single agent context, serving as the monolithic baseline. It runs on the native Codex CLI harness \cite{codex_cli}.

\textbf{Task Decomposition and Agent Generation (TDAG)} \cite{tdag} is a multi-agent framework built on dynamic task decomposition and agent generation. It decomposes complex tasks into smaller subtasks and assigns each subtask to a specifically generated subagent, thereby improving adaptability in diverse and unpredictable real-world tasks. The decomposition is not static: subtasks are dynamically adjusted based on the outcomes of preceding tasks, allowing the framework to refine the task structure as execution unfolds. Successful execution patterns are further summarized into reusable skills for future reference.

\textbf{DynTaskMAS} \cite{taskdynmas} is a dynamic task graph-driven framework for asynchronous and parallel LLM-based multi-agent systems. It orchestrates asynchronous and parallel operations through dynamic task graphs, and comprises four key components: a Dynamic Task Graph Generator that decomposes complex tasks while maintaining logical dependencies, an Asynchronous Parallel Execution Engine that optimizes resource utilization through efficient task scheduling, a Semantic-Aware Context Management System that enables efficient information sharing among agents, and an Adaptive Workflow Manager that dynamically optimizes system performance. This design supports scalable, high-performance multi-agent execution on complex, dynamic tasks.

\textbf{Graph-of-Agents (GoA)} \cite{goa} models multi-agent LLM communication through a graph-based framework. It first samples relevant agents from a candidate pool, then constructs edges between them by evaluating the relevance of their responses to one another, and finally aggregates the refined responses via graph-based pooling. Message passing proceeds in two stages, source-to-target followed by target-to-source, allowing highly relevant agents to refine the responses of less relevant ones and vice versa. To ensure consistency of the base model across all evaluated methods, we do not adopt GoA's original multi-source LLM design, in which different agents are backed by different foundation models. Instead, we instantiate all agents with the same foundation LLM and differentiate them by assigning distinct roles, so that the comparison isolates the effect of orchestration rather than that of heterogeneous model capabilities.

\subsection{Detailed Experimental Results}
\label{app:exp_details}

\subsubsection{Failure Pattern Analysis}
\label{app:failure_taxonomy}

To understand where and how multi-agent collaboration fails during execution, we conduct a coarse-grained semantic analysis. Specifically, this analysis focuses exclusively on the execution trajectories generated by our main experiments where DeepSeek-V4-Pro serves as the foundation model. We collect trajectories whose evaluation scores are significantly below the benchmark average and semantically analyze their execution traces, categorizing the dominant failures into three recurring types:
\begin{itemize}
    \item Task Understanding and Execution Failure TUEF: The agent fails to correctly parse the user's intent or environment constraints, leading to the execution of irrelevant or erroneous commands.
    \item Incomplete Workflow and Premature Termination IWPT: The agent understands the task but stops before the workflow is fully validated or the resolution criterion is satisfied. This often happens because the agent loses track of the full tool chain as the context length increases.
    \item Local Deadlock LD: The agent enters a state of cyclic dependencies, preventing any global progress. In multi-agent settings, this frequently arises when inter-agent communication or resource conflicts create dependency loops.
\end{itemize}

\begin{figure*}[htbp]
    \centering
    \begin{subfigure}[b]{0.48\textwidth}
        \centering
        \includegraphics[width=\linewidth]{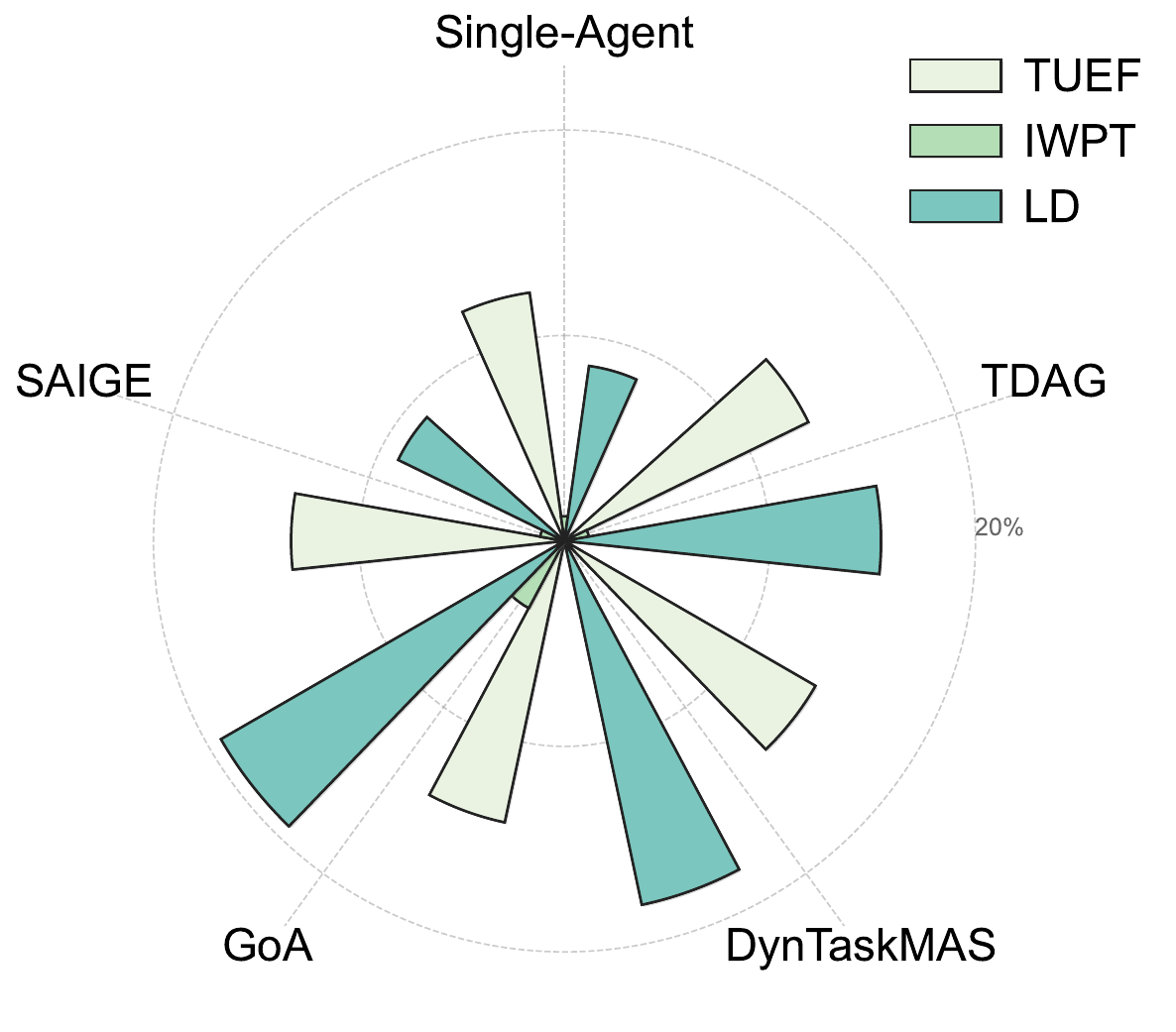}
        \caption*{(a) Terminal Bench 2.1}
        \label{fig:fail_terminal}
    \end{subfigure}
    \hfill
    \begin{subfigure}[b]{0.48\textwidth}
        \centering
        \includegraphics[width=\linewidth]{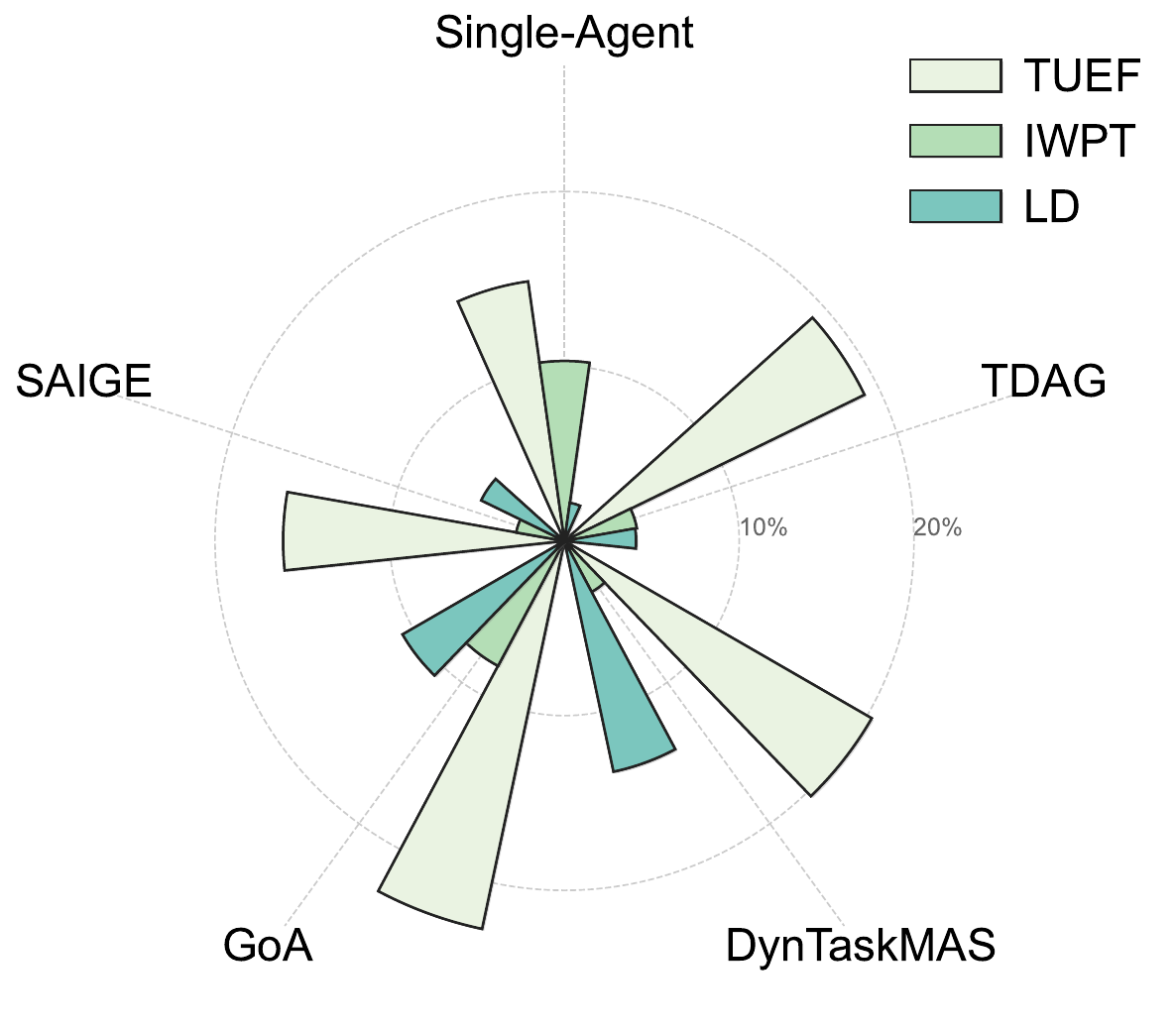}
        \caption*{(b) NL2Repo Bench}
        \label{fig:fail_nl2repo}
    \end{subfigure}
    
    \vspace{1em}
    
    \begin{subfigure}[b]{0.48\textwidth}
        \centering
        \includegraphics[width=\linewidth]{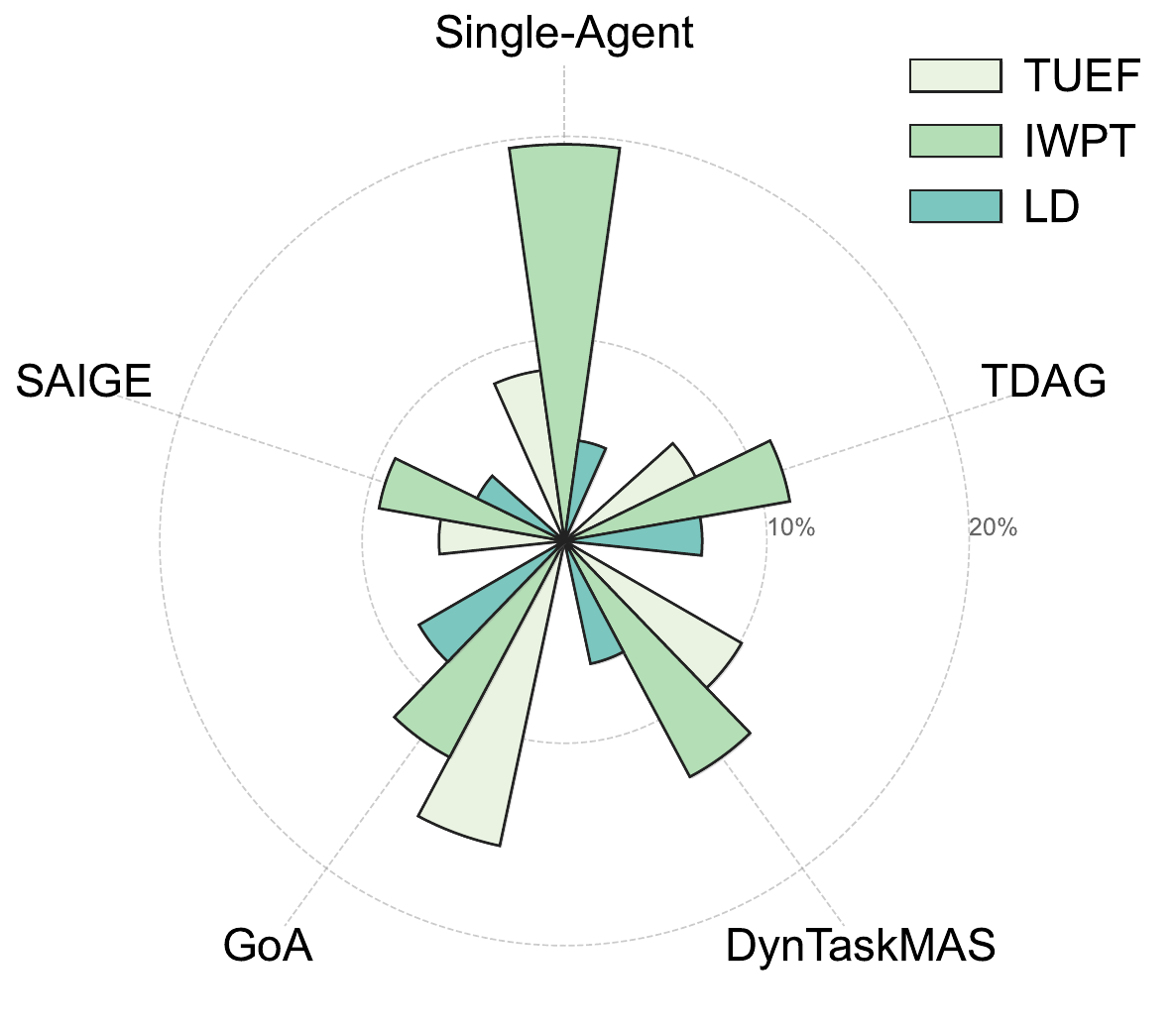}
        \caption*{(c) Deep Research Bench II}
        \label{fig:fail_drbii}
    \end{subfigure}
    \hfill
    \begin{subfigure}[b]{0.48\textwidth}
        \centering
        \includegraphics[width=\linewidth]{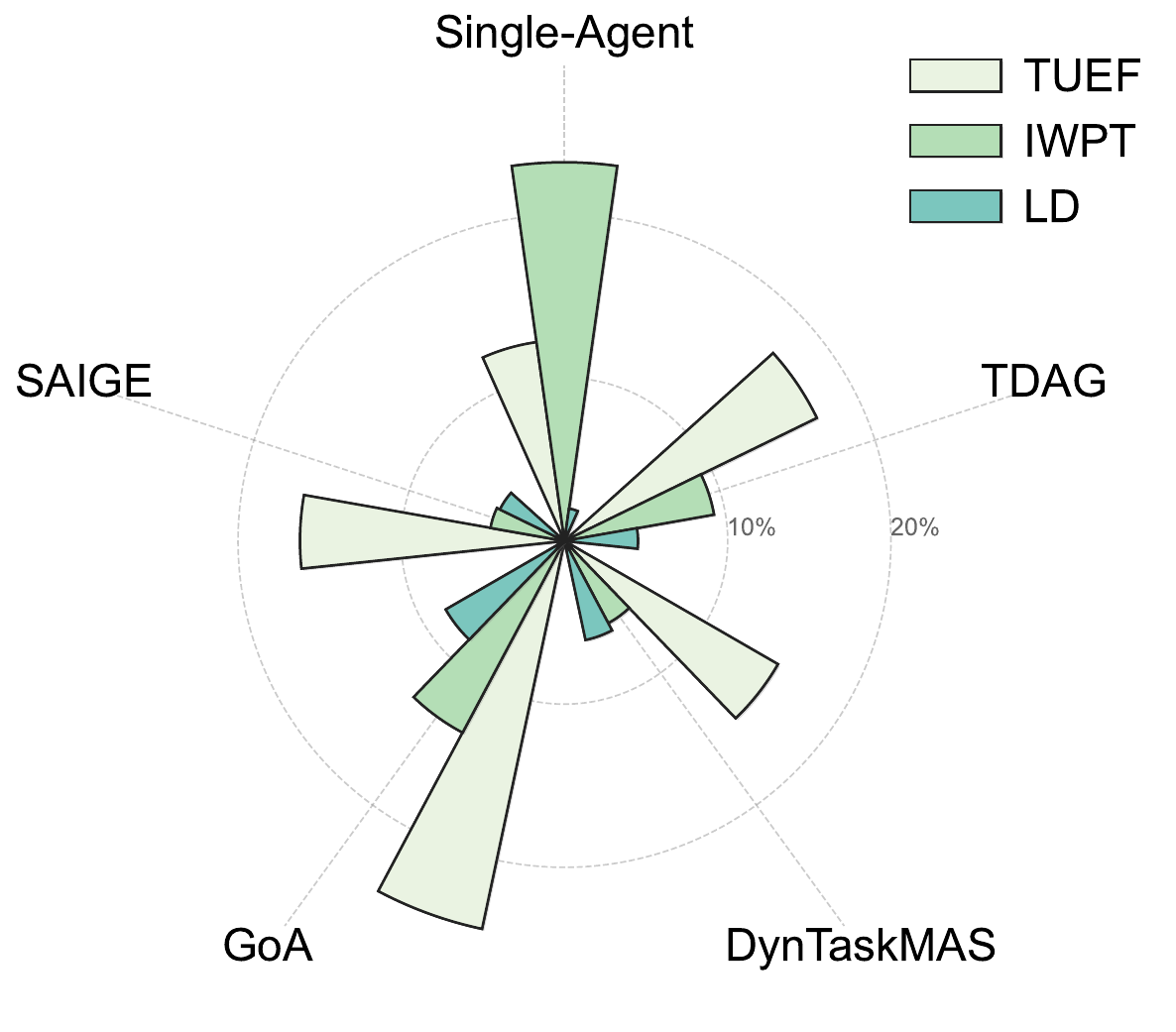}
        \caption*{(d) AgentIF-OneDay}
        \label{fig:fail_agentif}
    \end{subfigure}
    
    \caption{Distribution of failure types across four benchmarks, based on trajectories from DeepSeek-V4-Pro.}
    \label{fig:failure_taxonomy}
\end{figure*}

Figure~\ref{fig:failure_taxonomy} summarizes the distribution of these failure modes across the four benchmarks. On benchmarks with heavy sequential dependencies, such as Terminal Bench 2.1 and NL2Repo Bench, multi-agent methods are significantly more prone to Local Deadlock. The introduction of multiple agents adds coordination overhead and complex communication paths, which frequently lead to cyclic dependencies that hinder global progress. Conversely, on benchmarks with fewer inter-step dependencies, such as Deep Research Bench II and AgentIF-OneDay, single-agent methods more frequently exhibit Incomplete Workflow and Premature Termination. The accumulation of context progressively degrades the single agent's ability to track the full tool chain, causing it to stop prematurely before validating the workflow. These contrasting tendencies clarify why multi-agent collaboration is beneficial only when the task structure admits effective decomposition and why SAIGE, which grows its graph dynamically based on execution feedback, can avoid the Local Deadlock trap while mitigating the context degradation of single agents.

\subsubsection{Further Ablation Study}
\label{app:ablation_configs}

We further examine the ablation results using DeepSeek-V4-Pro under the five configurations for the multi-agent baselines. As shown in Figure~\ref{fig:ablation_all}, increasing the agent count or deepening the recursion level generally reduces the task score while increasing token consumption. Across both benchmarks and both baselines, i.e., TDAG and DynTaskMAS, the default configuration $(N=4, D=1)$ attains the most favorable overall trade-off between task performance and token cost, whereas the remaining configurations fall into the lower-right region relative to the default. This indicates that additional agents or deeper recursion introduce coordination overhead without improving problem-solving capacity.

\begin{figure*}[h]
    \centering
    \begin{subfigure}[b]{0.48\textwidth}
        \centering
        \includegraphics[width=\linewidth]{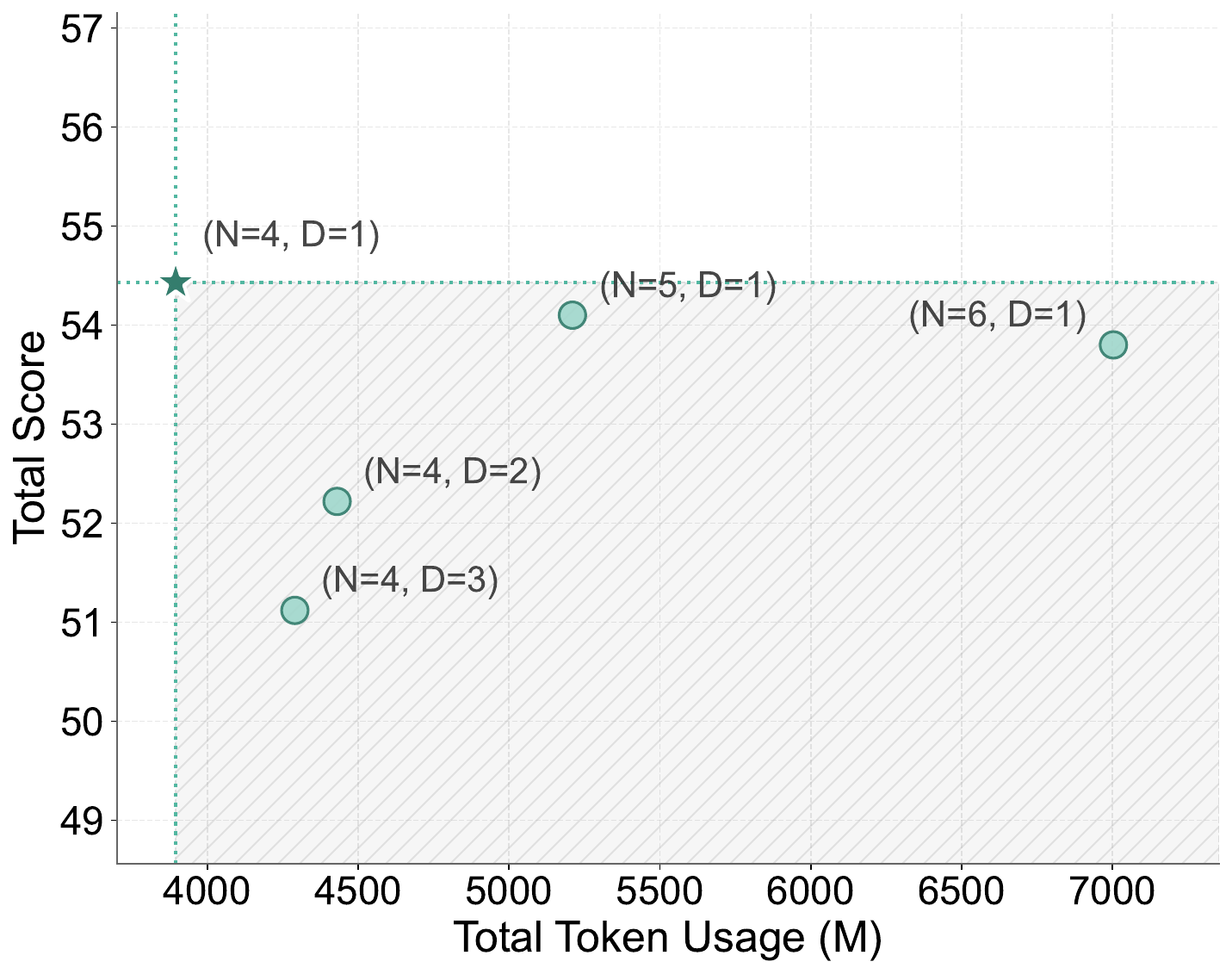}
        \caption*{(a) TDAG on Deep Research Bench II}
        \label{fig:ablation_tdag_drbii}
    \end{subfigure}
    \hfill
    \begin{subfigure}[b]{0.48\textwidth}
        \centering
        \includegraphics[width=\linewidth]{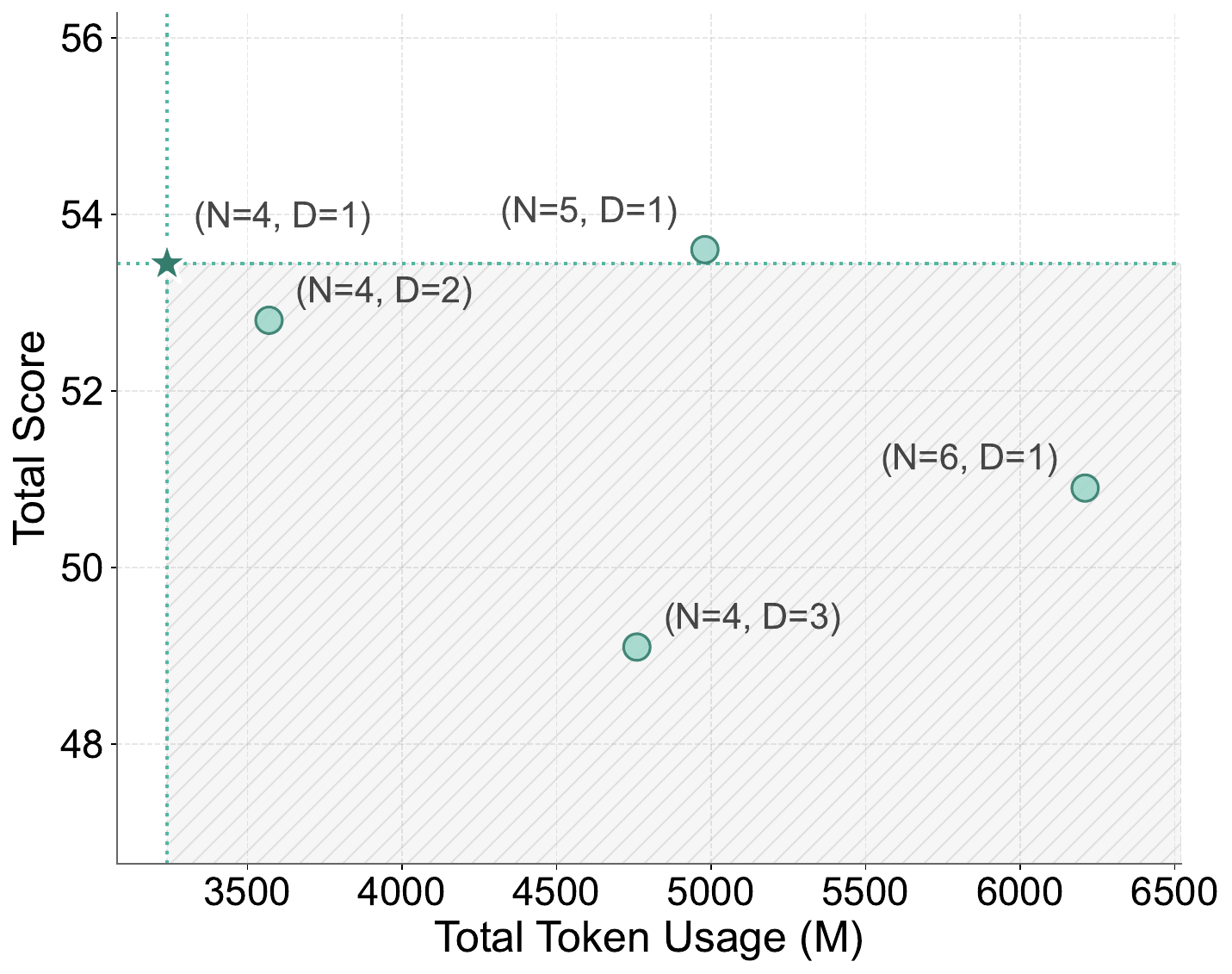}
        \caption*{(b) DynTaskMAS on Deep Research Bench II}
        \label{fig:ablation_dyntaskmas_drbii}
    \end{subfigure}
    
    \vspace{1em}
    
    \begin{subfigure}[b]{0.48\textwidth}
        \centering
        \includegraphics[width=\linewidth]{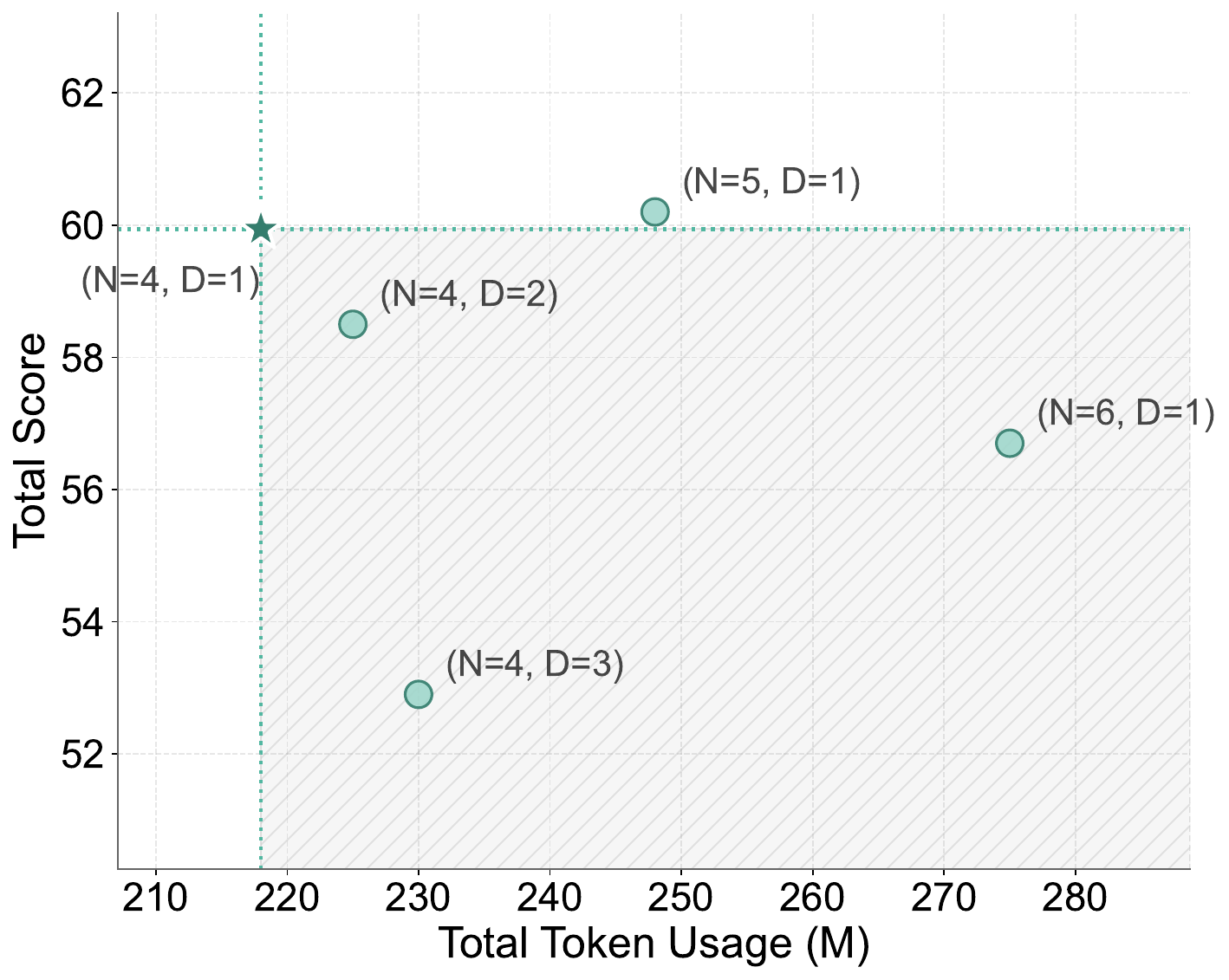}
        \caption*{(c) TDAG on AgentIF-OneDay}
        \label{fig:ablation_tdag_agentif}
    \end{subfigure}
    \hfill
    \begin{subfigure}[b]{0.48\textwidth}
        \centering
        \includegraphics[width=\linewidth]{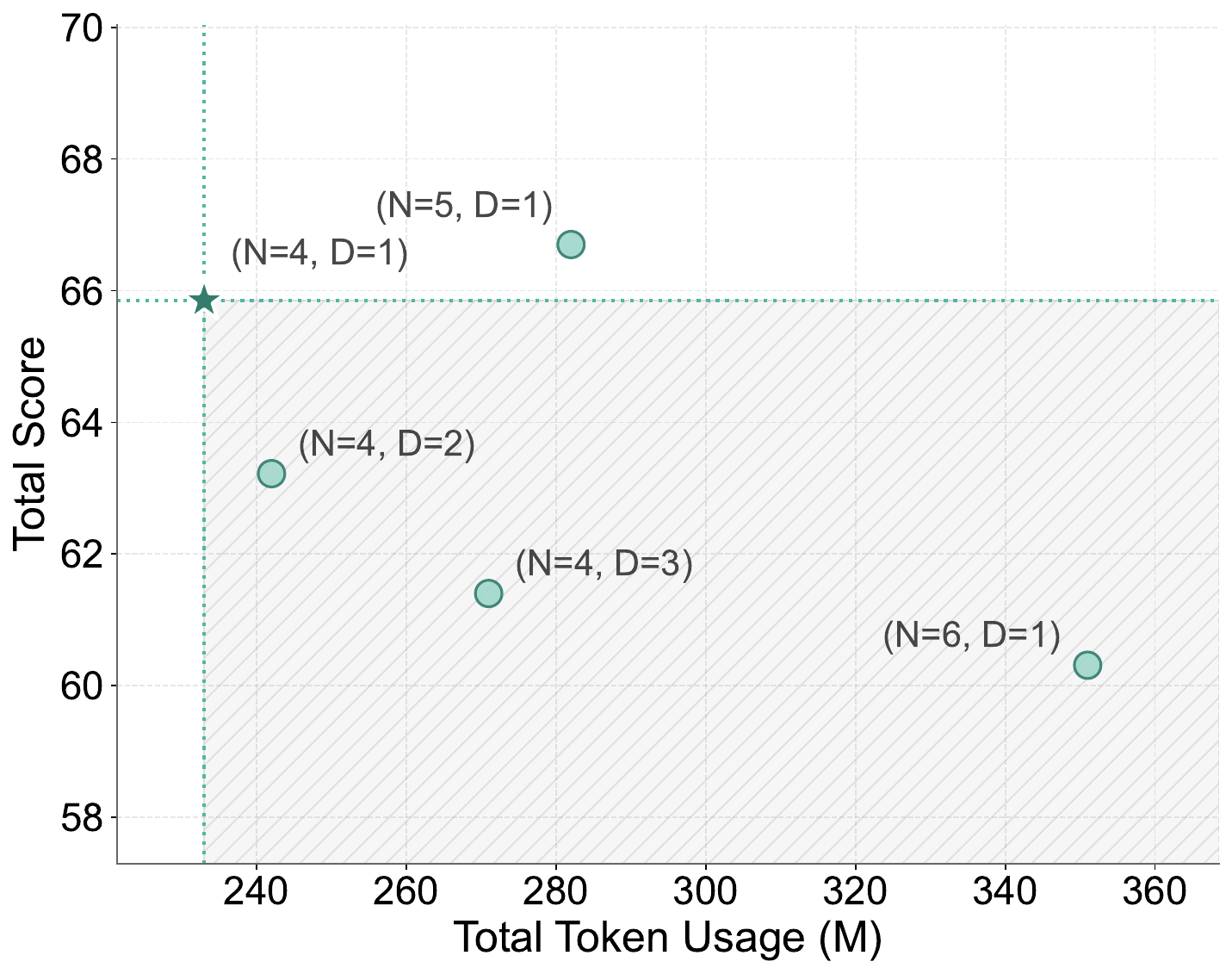}
        \caption*{(d) DynTaskMAS on AgentIF-OneDay}
        \label{fig:ablation_dyntaskmas_agentif}
    \end{subfigure}
    \caption{Ablation results under five configurations using DeepSeek-V4-Pro. Each point corresponds to one configuration.}
    \label{fig:ablation_all}
\end{figure*}

\section{Case Study}
\label{app:case_study}

\subsection{Task Dependency Patterns}

To illustrate the task dependency patterns that underlie our analysis, we select one representative single-agent trajectory from Terminal Bench 2.1 and one from AgentIF-OneDay using DeepSeek-V4-Pro, and construct their dependency matrices through semantic dependency analysis. Specifically, for each step $t$ in the trajectory, we take the history $\mathcal{H}_t = \{x_1, \ldots, x_{t-1}\}$ and assess, for each earlier step $i < t$, whether the subsequent execution would fail or require a retry if the output of step $i$ were removed from the history. A dependency is thus defined between step $t$ and step $i$ whenever removing $x_i$ leads to an erroneous or retried execution downstream. This yields a binary dependency matrix $A \in \{0, 1\}^{T \times T}$ with $A_{t,i} = 1$ if step $t$ depends on step $i$, and $A_{t,i} = 0$ otherwise.

\begin{figure*}[h]
    \centering
    \begin{subfigure}[b]{0.45\textwidth}
        \centering
        \includegraphics[width=\linewidth]{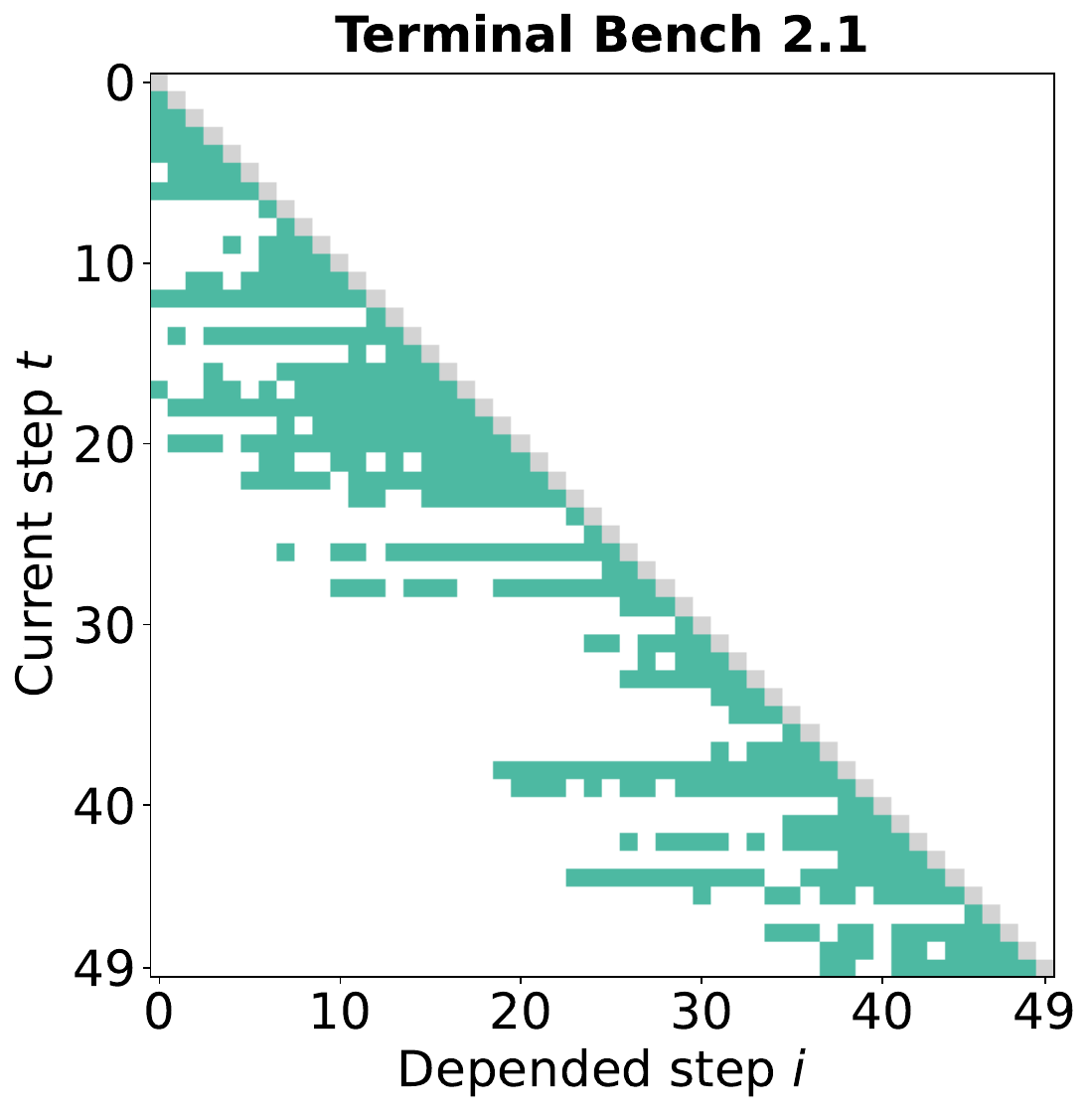}
    \end{subfigure}
    \hfill
    \begin{subfigure}[b]{0.45\textwidth}
        \centering
        \includegraphics[width=\linewidth]{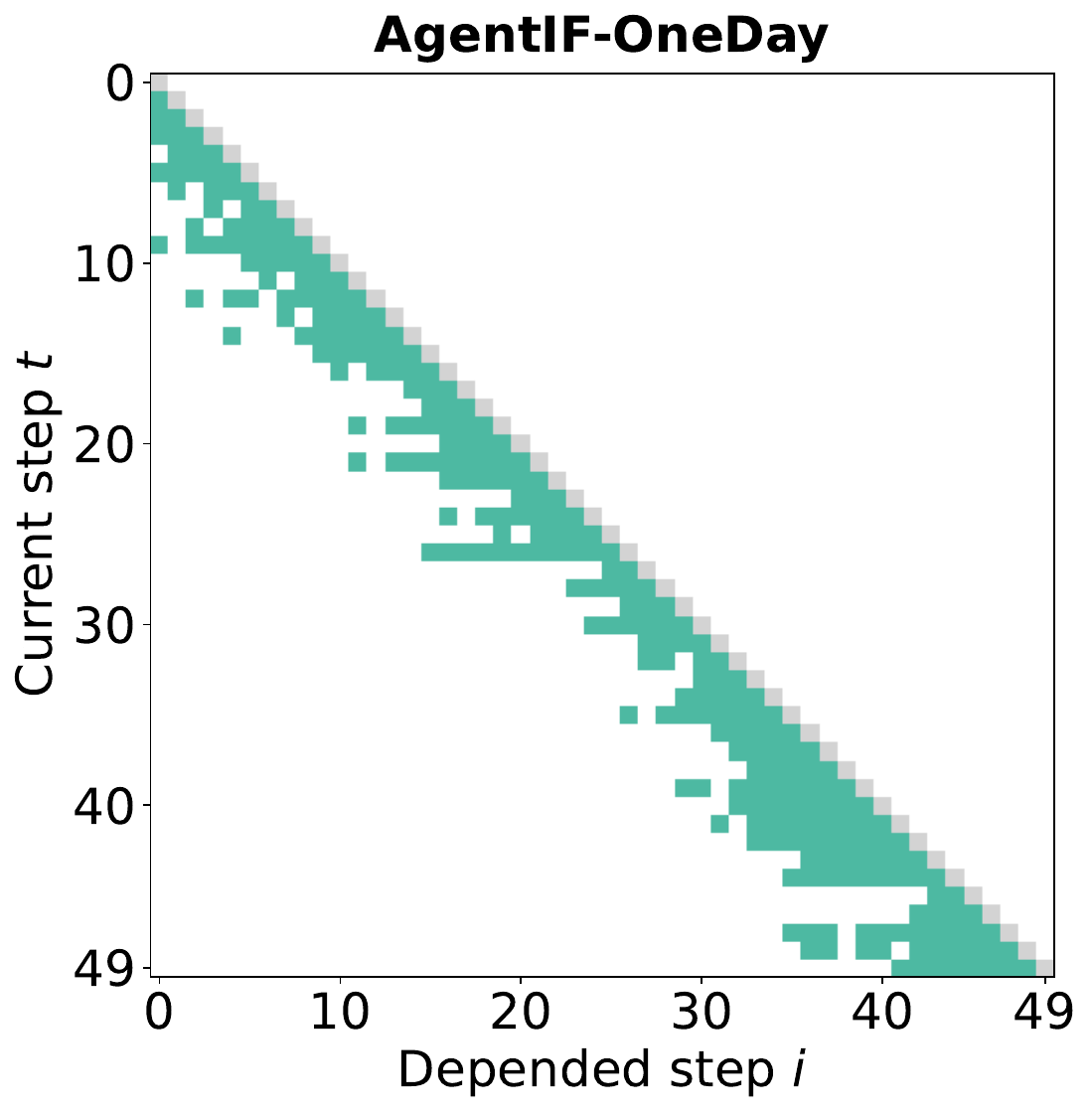}
    \end{subfigure}
    \caption{Dependency matrices of representative single-agent trajectories.}
    \label{fig:dependency_matrices}
\end{figure*}

Figure~\ref{fig:dependency_matrices} shows the resulting dependency matrices for the two trajectories. The Terminal Bench 2.1 trajectory exhibits a dense dependency pattern in which a non-negligible fraction of steps depend on distant predecessors, producing long-range dependencies that span a wide range of earlier steps. This is consistent with the long upstream-downstream chains induced by continuous tool invocations, file reads, and file writes, where later steps must consume the outputs of earlier ones and independent execution is rarely possible. The AgentIF-OneDay trajectory, by contrast, exhibits a sparse and localized pattern: non-zero entries concentrate near the diagonal, indicating that most steps depend only on their immediate predecessors and that large portions of the trajectory can proceed without conditioning on one another. This contrast directly reflects the two structural regimes identified in our analysis, and provides a step-level view of why multi-agent decomposition is constrained on densely dependent benchmarks while remaining feasible on loosely dependent ones.

\subsection{Bridge Edge Mis-Prediction: A Local Deadlock Case}

We illustrate the effect of bridge edge mis-prediction on a representative task from Terminal Bench 2.1, \texttt{make-doom-for-mips}, in which the agent must build a MIPS ELF that runs under a supplied JavaScript MIPS VM. The task contains a long upstream-downstream chain: the VM defines the syscall ABI and the memory layout, the build must target a compatible MIPS ABI, and the resulting ELF must be validated by running it inside the VM. A correct decomposition must preserve the bridge edges that connect these stages.

\begin{figure}[h]
\centering
\begin{minipage}[t]{0.48\linewidth}
\small
\begin{verbatim}
// Ground-truth dependency structure
1  inspect vm.js
2      -> identify syscall ABI
3      -> identify memory layout
4  inspect doomgeneric/
5      -> identify build targets
6  configure MIPS toolchain
7      -> select ABI / soft-float
8  build doomgeneric_mips
9  run under vm.js
10     -> observe failure
11 -> adjust build flags
12 -> rebuild
13 -> validate execution
\end{verbatim}
\end{minipage}
\hfill
\begin{minipage}[t]{0.48\linewidth}
\small
\begin{verbatim}
// Predicted dependency structure
1  inspect vm.js
2      -> identify syscall ABI
3  inspect doomgeneric/
4      -> identify build targets
5  configure MIPS toolchain
6      -> select ABI / soft-float
7  build doomgeneric_mips
8  run under vm.js
9      -> observe failure
10 -> adjust build flags
11 -> rebuild
12 -> validate execution
\end{verbatim}
\end{minipage}
\caption{Ground-truth and mis-predicted dependency structures on \texttt{make-doom-for-mips}.}
\label{fig:bridge_case}
\end{figure}

The error is one omitted bridge edge: the dependency from VM memory layout (step 3) to build configuration (step 6) in the ground-truth dependency structure. In the ground-truth structure, this edge carries the constraint that determines the correct ABI and memory model. In the predicted structure, the edge is absent.

This omission produces a local deadlock. The build agent configures the MIPS toolchain without the VM memory layout, while the validation agent repeatedly rebuilds and waits for a build that satisfies the omitted constraint. Neither agent can break the cycle: the build agent adjusts flags that are not the actual bottleneck, and the validation agent keeps issuing rebuild-and-run cycles.

Two costs follow. Communication overhead rises because repeated build-and-run messages do not converge, and redundant execution accumulates because each cycle reruns the same build and the same VM validation.

If the bridge edge is predicted correctly, the build configuration is conditioned on the VM memory layout and syscall ABI, and the correct ABI and soft-float configuration are selected before the first build. The validation stage then observes a consistent ELF, and the trajectory terminates after a small number of build-and-run cycles. This contrast shows why orchestration is naturally formulated as bridge edge prediction: an omitted bridge edge does not merely reduce context cost reduction, but can turn a convergent workflow into a local deadlock.

\section{Impact Statement}

This work asks when a task should be delegated to subagents and when it is better handled by a single agent. Viewed at the granularity of individual trajectories, its primary impact is conceptual and practical, as it reframes the question from how many agents to assemble to whether the task's dependency structure admits decomposition at all. Our trajectory-level analysis suggests that the agent orchestration pattern should match the bridge-edge structure of the task, and our experiments indicate that subagents tend to benefit sparse bridge-edge structures, whereas tightly coupled trajectories are often better served by a single agent. For practitioners, this points toward introducing collaboration selectively, guided by the dependency structure of the target task, instead of adopting it as a default design choice.

On the positive side, the proposed SAIGE mechanism spawns subagents only at predicted bridge edges and resolves their connections lazily through content-based information retrieval. Where the structure admits decomposition, this may help curb unnecessary context accumulation and support more reliable long-horizon execution, which is relevant to applications such as software engineering agents, repository-level code generation, and deep research assistants. By making the cost-performance trade-off of decomposition explicit, our analysis may also encourage more disciplined evaluation and reporting practices in this area. We note that these benefits are conditional: on tightly coupled tasks, the same mechanism offers little advantage over a single-agent harness, and our ablations suggest that scaling the agent pool or recursion depth does not reliably help.

On the risk side, more capable and autonomous agentic systems can be misused for harmful purposes, including automated exploitation of software vulnerabilities, unauthorized access to computing resources, or large-scale generation of misleading content. The benchmarks used in this work are designed for research evaluation and operate in controlled environments, but the underlying techniques are not inherently restricted to benign settings. Decomposition may also introduce failure modes that are harder to audit, such as local deadlocks, mispredicted bridge edges, and opaque coordination among agents, which can complicate debugging, accountability, and oversight.

\end{document}